%% file: main.tex
\documentclass{article} %
\usepackage{newtxtext}
\usepackage[left=1.25in, top=1in, bottom=1in, right=1.25in]{geometry}

\usepackage[utf8]{inputenc} %
\usepackage{url}            %
\usepackage{booktabs}       %
\usepackage{amsfonts}       %
\usepackage{nicefrac}       %
\usepackage{xcolor}
\usepackage{amsmath,amssymb,amsthm}
\usepackage{natbib}
\usepackage{graphicx}
\usepackage{caption}
\usepackage{soul}
\usepackage{enumitem}
\usepackage{subcaption}
\usepackage{pifont}%
\usepackage{makecell}
\usepackage[colorlinks=true,linkcolor=magenta,citecolor=blue]{hyperref}%
\usepackage{cleveref}

\newcommand{\cmark}{\ding{51}}%
\newcommand{\xmark}{\ding{55}}%
\include{math_commands}
\newtheorem{theorem}{Theorem}[section]
\newtheorem{proposition}{Proposition}[section]
\newtheorem{lemma}{Lemma}[section]
\newtheorem{definition}{Definition}[section]

\newtheorem{question}{Question}
\newtheorem{condition}{Condition}

\newtheorem{corollary}{Corollary}[section]

\title{
Sharpness Minimization Algorithms Do Not Only Minimize Sharpness To Achieve Better Generalization
}

\date{}
\author{Kaiyue Wen \\
Tsinghua University\\
\texttt{wenky20@mails.tsinghua.edu.cn}
\and
Zhiyuan Li \\
Stanford University\\
\texttt{zhiyuanli@stanford.edu} \\
\and
Tengyu Ma  \\
Stanford University\\
\texttt{tengyuma@stanford.edu} \\
}

\begin{document}

\maketitle

\begin{abstract}
Despite extensive studies, the underlying reason as to why overparameterized neural networks can generalize remains elusive. 
Existing theory shows that common stochastic optimizers prefer flatter minimizers of the training loss, 
and thus a natural potential explanation is that flatness implies generalization. 
This work critically examines this explanation. Through theoretical and empirical investigation, we identify the following three scenarios for two-layer ReLU networks:
(1) flatness provably implies generalization; (2)  there exist non-generalizing flattest   
models  and sharpness minimization algorithms fail to generalize, and 
(3) perhaps most surprisingly, there exist non-generalizing  flattest 
models, but sharpness minimization algorithms still generalize. Our results suggest that the relationship between sharpness and generalization 
subtly depends on the data distributions and the model architectures and sharpness minimization algorithms do not only 
minimize sharpness to achieve better generalization. This calls for the search for other explanations for the generalization of over-parameterized neural networks.
\end{abstract}
\input{intro_new}

\input{preliminary}

\input{Case1}

\input{Case2}

\input{Case3}
\input{conclusion}

\bibliography{all,ours}
\bibliographystyle{iclr2023_conference}

\appendix
\newpage {
\hypersetup{linkcolor=black}
\tableofcontents
}
\newpage
\input{appendix}

\end{document}

%% file: math_commands.tex
\usepackage{amsmath,amsfonts,bm}

\def\1{\bm{1}}

\DeclareMathAlphabet{\mathsfit}{\encodingdefault}{\sfdefault}{m}{sl}
\SetMathAlphabet{\mathsfit}{bold}{\encodingdefault}{\sfdefault}{bx}{n}

\def\gF{{\mathcal{F}}}

\def\gR{{\mathcal{R}}}

\newcommand{\R}{\mathbb{R}}

\newcommand{\Var}{\mathrm{Var}}

\DeclareMathOperator*{\argmax}{arg\,max}
\DeclareMathOperator*{\argmin}{arg\,min}

\DeclareMathOperator{\sign}{sign}

\DeclareMathOperator{\Exp}{\mathop \mathbb{E}}

\newcommand*{\one}{{\bm 1}}

\newcommand*{\E}{\mathrm{E}}

\newcommand{\eps}{\epsilon}

\newcommand{\rada}{\mathcal{R}}
\newcommand{\model}{f}
\newcommand{\para}{\theta}

\newcommand{\distr}{\mathcal{P}_{\mathrm{xor}}}
\newcommand{\feature}{x}
\newcommand{\target}{y}
\newcommand{\loss}{\ell}
\newcommand{\Loss}{L}
\newcommand{\mseLoss}{L^{\mathrm{mse}}}
\newcommand{\logLoss}{L^{\mathrm{log}}}
\newcommand{\Trace}{\mathrm{Tr}}
\newcommand{\ReLU}{ \mathsf{relu}}
\newcommand{\relu}{\ReLU}
\newcommand{\mseloss}{\ell_{\mathrm{mse}}}
\newcommand{\logistic}{\ell_{\mathrm{logistic,p}}}

\newcommand{\simplebn}{\mathsf{SBN}}

\newcommand{\weight}{W}
\newcommand{\bias}{b}

\newcommand{\sgn}{\mathrm{sign}}
\newcommand{\mlpnobias}{\textsf{2-MLP-No-Bias}}
\newcommand{\mlpbias}{\textsf{2-MLP-Bias}}
\newcommand{\mlpsimplebn}{\textsf{2-MLP-Sim-BN}}
\newcommand{\mlpsimpleln}{\textsf{2-MLP-Sim-LN}}
\newcommand{\modelnobias}{f^{\mathrm{nobias}}}
\newcommand{\modelbias}{f^{\mathrm{bias}}}
\newcommand{\modeldeepbias}{f^{\mathrm{bias,D}}}
\newcommand{\modelbn}{f^{\mathrm{sbn}}}
\newcommand{\modelln}{f^{\mathrm{sln}}}
\newcommand{\tracehessian}{trace of Hessian of the loss}
\newcommand{\sharpness}[1]{\Trace\left(\nabla^2 \Loss(#1)\right)}
\newcommand{\deepmlpbias}{\textsf{MLP-Bias}}
\newcommand{\datad}{\mathcal{D}}
\newcommand{\VC}{\mathrm{VC}}

%% file: intro_new.tex
\section{Introduction}

It remains mysterious why stochastic optimization methods such as stochastic gradient descent (SGD) can find generalizable models even when the architectures are overparameterized~\citep{zhang2016understanding,gunasekar2017implicit,li2017algorithmic,soudry2018implicit,woodworth2020kernel}. Many empirical and theoretical studies suggest that generalization is correlated with or guaranteed by the flatness of the loss landscape at the learned model~\citep{hochreiter1997flat,keskar2016large,dziugaite2017computing,jastrzkebski2017three,neyshabur2017exploring,wu2018sgd,jiang2019fantastic,blanc2019implicit,wei2019data,wei2019improved,haochen2020shape,foret2021sharpnessaware,damian2021label,li2021happens,ma2021linear,ding2022flat,nacson2022implicit,wei2022statistically,lyu2022understanding,norton2021diametrical,wu2023implicit}. Thus, a natural theoretical question is 

\setcounter{question}{-1}
\begin{question}
\label{hypo:1}
 Does the flatness of the minimizers always correlate with the generalization capability?
\end{question}
The answer to the question turns out to be false. First, \citet{dinh2017sharp} theoretically construct \textit{very sharp} networks with good generalization. Second,  recent empirical results~\citep{andriushchenko2023modern} find that sharpness may not have a strong correlation with test accuracy for a collection of modern architectures and settings, partly due to the same reason---there exist sharp models with good generalization.
We note that, technically speaking, Question~\ref{hypo:1} is ill-defined without specifying the collection of models on which the correlation is evaluated. However, those sharp but generalizable models appear to be the main cause for the non-correlation. 

Observing the existing theoretical and empirical evidence, it is natural to ask the one-side version of \Cref{hypo:1}, where we are only interested in whether sharpness implies generalization but not vice versa.

\begin{question}
\label{hypo:flat}
 Do all the flattest neural network minimizers generalize well?
\end{question}

Though there are some theoretical works that answer~\Cref{hypo:flat} affirmatively for simplified linear models~\citep{li2021happens,ding2022flat,nacson2022implicit,gatmiry2023inductive}, the answer to~\Cref{hypo:flat} for standard neural networks remains unclear. Those theoretical results linking generalization to sharpness for more general architectures typically also involve other terms in generalization bounds, such as parameter dimension or norm~\citep{neyshabur2017exploring,foret2021sharpnessaware,wei2019data,wei2019improved,norton2021diametrical}, thus do not answer \Cref{hypo:flat} directly.

Our first contribution is a theoretical analysis showing that the answer to~\Cref{hypo:flat} can be \textbf{false}, even for simple architectures like 2-layer ReLU networks. Intriguingly, we also find that the answer to~\Cref{hypo:flat} subtly depends on the architectures of neural networks. For example, simply removing the bias in the first layer turns the aforementioned negative result into a positive result, as also shown in the Theorem 4.3 of \citet{wu2023implicit} (that the authors only came to be aware of after putting this work online).

More concretely, we show that for the 2 parity xor problem with mean square loss and with data sampled from hypercube $\{-1,1\}^d$, all flattest 2-layer ReLU neural networks without bias provably generalize. However, when bias is added, for the same data distribution and loss function, there exists a flattest minimizer that fails to generalize for every unseen data. Since adding bias in the first layer can be interpreted as appending a constant input feature, this result suggests that the generalization of the flattest minimizer is sensitive to both network architectures and data distributions.

Recent theoretical studies \citep{wu2018sgd,blanc2019implicit,damian2021label,li2021happens,arora2022understanding,wen2022does,nacson2022implicit,lyu2022understanding,bartlett2022dynamics,li2022fast} also show that optimizers including SGD with large learning rates or label noise and Sharpness-Aware Minimization (SAM, \citet{foret2021sharpnessaware}) may implicitly regularize the sharpness of the training loss landscape. These optimizers are referred to as \emph{sharpness minimization algorithms} in this paper.
Because ~\Cref{hypo:flat} is not always true, it is then natural to hypothesize that sharpness-minimization algorithms will fail for architectures and data distributions where~\Cref{hypo:flat} is not true.

\begin{question}
\label{hypo:work}
Will sharpness minimization algorithm fail to generalize when there exist non-generalizing flattest minimizers?
\end{question}

A priori, the authors were expecting that the answer to~\Cref{hypo:work} is affirmative, which means that a possible explanation is that the sharpness minimization algorithm works if and only if for certain architecture and data distribution, \Cref{hypo:flat} is true. However, surprisingly, we also answer this question negatively for some architectures and data distributions. In other words, we found that sharpness-minimization algorithms can still generalize well even when the answer to~\Cref{hypo:flat} is false. The result is consistent with our theoretical discovery that for many architectures, there exist both non-generalizing and generalizing flattest minimizers of the training loss. We show empirically that sharpness-minimization algorithms can find different types of minimizers for different architectures.

\begin{table}[t]
    \centering
    \begin{tabular}{|l|c|c|}
    \hline 
    Architecture & \makecell{All Flattest Minimizers \\ Generalize Well.}& \makecell{Sharpness Minimization \\
    Algorithms Generalize.} \\
    \hline
    2-layer w/o Bias & \cmark \ (\Cref{thm:sharpness_generalization_complexity}) & \cmark \\
    2-layer w/ Bias & \xmark \ (\Cref{thm:sharpness_counter_example}) & \xmark \\
    2-layer w/ simplified BatchNorm & \cmark \ (\Cref{thm:sharpness_generalization_feature}) & \cmark \\
    2-layer w/ simplified LayerNorm & \xmark \ (\Cref{thm:sharpness_counter_example_ln}) & \cmark \\
    \hline
    \end{tabular}
    \vspace{0.3cm}
    \caption{\textbf{Overview of Our Results.} Each row in the table corresponds to one architecture. The second column indicates whether all flattest minimizers of training loss generalize well. \cmark\ indicates that all (near) flattest minimizers of training loss provably generalize well and \xmark\ indicates that there provably exists flattest minimizers that generalize poorly.
    The third column indicates whether the sharpness minimization algorithms generalize well in our experiments. Results in row $2$ and $4$ deny~\Cref{hypo:flat} and \Cref{hypo:work} respectively.
    }
    \label{table:overview}
    \end{table}

Our results are summarized in~\Cref{table:overview}. 
We show through theoretical and empirical analysis that the relationship 
between sharpness and generalization can fall into three different regimes depending on the architectures and distributions. 
The three regimes include:

\begin{itemize}[leftmargin=*]
\setlength\itemsep{-0.1ex}
\vspace{-0.2cm}
   \item \textbf{Scenario 1.} Flattest minimizers of training loss provably generalize and sharpness minimization 
    algorithms find generalizable models. This regime (\Cref{thm:sharpness_generalization_complexity,thm:sharpness_generalization_feature}) includes 2-layer ReLU MLP without bias and 2-layer ReLU MLP with a simplified BatchNorm (without mean subtraction and bias). We answer both the \Cref{hypo:flat} and~\Cref{hypo:work} affirmatively in this scenario.\footnote{The condition for \Cref{hypo:work} is not satisfied and thus the answer to \Cref{hypo:work} is affirmative.}
  \item \textbf{Scenario 2.} There exists a flattest minimizer that has the worst generalization over all minimizers. Also, sharpness minimization algorithms fail to find generalizable models. This regime includes $2$ layer ReLU MLP with bias. We deny~\Cref{hypo:flat} while affirm~\Cref{hypo:work} in this scenario.
   \item \textbf{Scenario 3.} There exist flattest minimizers that do not generalize but the
    sharpness minimization algorithm still finds the generalizable flattest model empirically. This regime includes 2-layer ReLU MLP with a simplified LayerNorm (without mean subtraction and bias). In this scenario, the
    sharpness minimization algorithm relies other unknown mechanisms beyond minimizing sharpness to find a generalizable model. We deny both~\Cref{hypo:flat} and~\Cref{hypo:work} in this scenario.
\end{itemize}

%% file: preliminary.tex
\section{Setup}
\paragraph{Rademacher Complexity.} Given $n$ data $S = \{x_i\}_{i=1}^n$, the \emph{empirical Rademacher complexity} of function class $\gF$ is defined as 
$\gR_S(\gF) = \frac{1}{n}\Exp_{\epsilon \sim \{\pm 1\}^n} \sup_{f\in \gF}\sum_{i=1}^n  \epsilon_i f(x_i)$.
\paragraph{Architectures.} As summarized in~\Cref{table:overview}, we will consider multiple network architectures and discuss how architecture influences the relationship between 
sharpness and generalization. For each model $\model_\para$ parameterized by $\para$, we will use $d$ to denote the input dimension and $m$ to denote the network width. We will now describe the architectures in detail.
\\
\\
\mlpnobias.$\modelnobias_\para(x) = \weight_2 \ReLU \left( \weight_1 x \right)$ with $\theta = (\weight_1, \weight_2)$.
\\
\\
\mlpbias. $\modelbias_\para(x) = \weight_2 \ReLU \left( \weight_1 x + \bias_1 \right)$ with $\theta = (\weight_1, \bias_1, \weight_2)$. We additionally define \deepmlpbias~as $\modeldeepbias_\para(x) = \weight_D \ReLU \cdots \weight_2 \ReLU \left( \weight_1 x + \bias_1 \right)$,
\\
\\
\mlpsimplebn. $\modelbn_\para(x, \{\feature_i\}_{i \in [n]})  = \weight_2 \simplebn_\gamma\left(\ReLU\left(\weight_1 x + \bias_1\right), \left\{\ReLU\left(\weight_1 x_i+ \bias_1\right)\right\}\right)$
, where the simplified BatchNorm $\simplebn$ is defined as $\forall m, n \in N, \forall i \in [n], x, x_i \in \R^m, j \in [m], \simplebn_{\gamma}(x, \{x_i\}_{i \in [n]})[j] = \gamma x[j] / \left( \sum_{i=1}^n (x_{i}[j])^2 / n\right)^{1/2}$ and $\theta =  (\weight_1, \bias_1, \gamma, \weight_2)$.
\\
\\
\mlpsimpleln.
$\modelln_\para(x) = \weight_2 \frac{\ReLU(\weight_1 x + \bias_1)}{\max\{\|\ReLU(\weight_1 x + \bias_1)\|_2, \epsilon \}}$ where $\epsilon$ is a sufficiently small positive constant. 
\\

Surprisingly, our results show that the relationships between sharpness and generalization are strikingly different among these simple yet similar architectures.

\paragraph{Data Distribution.} We will consider a simple data distribution as our testbed. 
Data distribution $\distr$ is a joint distribution over data point $\feature$ and label $\target$. The data point is sampled uniformly from the hypercube $\{-1,1\}^d$ and the label satisfies $y = x[1]x[2]$. Many of our results, including our generalization bound in~\Cref{sec:scenario1} and experimental observations can be generalized to broader family of distributions (\Cref{sec:appendix_theory}).

\paragraph{Loss.} We will 
use mean squared error $\mseloss$ for training and denote the training loss as $\Loss$. In~\Cref{sec:appendix_theory}, we will show that all our theoretical results and empirical observations hold for logistic loss with label smoothing probability $p > 0$. We will also consider zero one loss $\Pr(y\model_{\para}(x) > 0)$ for evaluating the model. We will use interpolating model to denote the model with parameter $\para$ that minimizes $\Loss$.

\begin{definition}[Interpolating Model]
    A model $f_\para$ interpolates the dataset $\{(\feature_i, \target_i)\}_{i=1}^n$ if and only if  
    $\forall i, f_\para(\feature_i) = \target_i$.
\end{definition}

\paragraph{Sharpness.}
Our theoretical analysis focuses on understanding the \emph{sharpness} of the trained models. Precisely, for a model ${\model}_\para$ parameterized by $\para$, a dataset $\{({\feature}_i, {\target}_i )\}_{i=1}^n$ and loss function $\loss$, we will use the trace of Hessian of loss function, $\Trace(\nabla^2 \Loss(\para))$ to measure how sharp the loss is at $\para$, which is a proxy for the sharpness along a random direction~\citep{wen2022does}, or equivalently, the expected increment of loss under a random gaussian perturbation~\citep{foret2021sharpnessaware,orvieto2022explicit} .

$\Trace(\nabla^2 L(\theta))$ is not the only  choice for defining sharpness, but theoretically many sharpness minimization algorithms have been shown to minimize this term over interpolating models. In particular,
under the assumptions that the minimizer of the training loss form a smooth manifold~\citet{cooper2018loss,fehrman2020convergence}, Sharpness-Aware Minimization (SAM) \citep{foret2021sharpnessaware} with batch size $1$ and sufficiently small learning rate $\eta$ and perturbation radius $\rho$~\citep{wen2022does,bartlett2022dynamics}, or 
Label Noise SGD with sufficiently small learning rate $\eta$~\citep{blanc2019implicit,damian2021label,li2021happens}, 
prefers interpolating models with small \tracehessian.    
Hence, we choose to analyze \tracehessian~and will use SAM with batch size $1$ (we denote it by 1-SAM) as our sharpness minimization algorithm in our experiments.

\paragraph{Notations.}  We  use $\mathrm{Tr}$ to denote the trace of a matrix and $x[i]$ to denote the value of the $i$-th coordinate of vector $x$. We will use $\odot$ to represent element-wise product. We use $\one$ as the (coordinate-wise) indicator function, for example, $\one\left[x > 0\right]$ is a vector of the same length as $x$ whose $j$-th entry is $1$ if $x[j] > 0$ and $0$ otherwise. We 
will use $\tilde O(x)$ to hide logarithmic multiplicative factors.

%% file: Case1.tex
\section{Scenario I: All Flattest Models Generalize}
\label{sec:scenario1}
\begin{figure}
    \centering
    \begin{subfigure}[t]{0.45\textwidth}
        \centering
        \includegraphics[width=\textwidth]{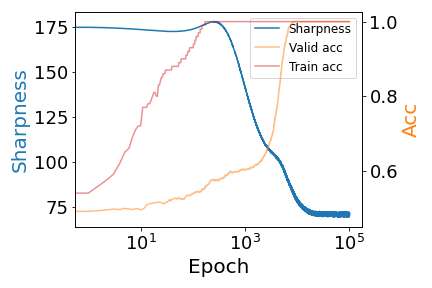}
        \caption{Baseline}
        \label{fig:classification_no_bias_baseline}
    \end{subfigure}
    \begin{subfigure}[t]{0.45\textwidth}
       \centering
       \includegraphics[width=\textwidth]{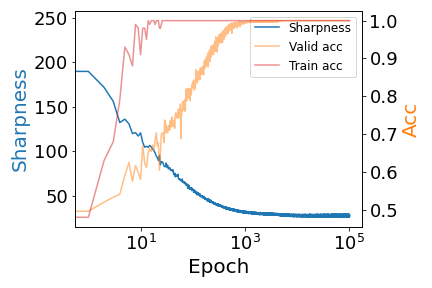}
       \caption{1-SAM}
       \label{fig:classification_no_bias_sam}
   \end{subfigure}
\caption{
\textbf{Scenario I.} We train a 2-layer MLP with ReLU activation without bias using gradient descent with weight decay and 1-SAM on $\distr$ with dimension $d = 30$ and training set size $n = 100$. In both cases, the model reaches perfect generalization. Notice that although weight decay doesn't explicitly regularize model sharpness, the flatness of the model decreases through training, which is consistent with our~\Cref{lem:norm_bound} relating sharpness to the norm of the weight.}
\label{fig:classification_no_bias}
\end{figure}

\subsection{Flattest models provably generalize}

When the architecture is \mlpnobias, we will show that the flattest models can provably generalize, hence answering~\Cref{hypo:flat} affirmatively for this architecture and data distribution $\distr$.

\begin{theorem}
\label{thm:sharpness_generalization_complexity}
    For any $ \delta \in (0,1)$ and input dimension $d$, for $n = \Omega\left(d \log\left(\frac{d}{\delta}\right)\right)$, with probability at least $1 - \delta$ over the random draw of training set $\{(x_i,y_i)\}_{i=1}^n$ from $\distr^{n}$, let $L(\theta)\triangleq \frac{1}{n}\sum_{i=1}^n \mseloss(\modelnobias_\theta(x_i),y_i)$ be the training loss for \mlpnobias, it holds that for all $\para^*\in  \argmin_{L(\para) = 0}\mathrm{Tr}\left( \nabla^2 L \left ( \theta \right) \right)$, we have that
    \begin{align}
        &\E_{x, y \sim \distr} \left[\mseloss\left(\modelnobias_{\para^*}\left(x\right),y\right) \right ]\le \tilde O\left(d/n\right). \notag 
        \end{align}
\end{theorem}
\Cref{thm:sharpness_generalization_complexity} shows that for $\distr$, flat models can generalize under almost linear sample complexity with respect to the input dimension. We note that~\Cref{thm:sharpness_generalization_complexity} implies that $\Pr_{x, y \sim \distr}\left[\modelnobias_{\para^*}(x) y > 0\right] \le \tilde O\left(d/n\right).$ because if $\modelnobias_{\para^*}(x) y \le 0$, it holds that $\mseloss\left(\modelnobias_{\para^*}\left(x\right),y\right) \ge 1$. This shows that the model can classify the input with high accuracy.
The major proof step is relating 
sharpness
to the norm of the weight itself.

\begin{lemma}
\label{lem:norm_bound}
 Define $\Theta_C \triangleq  \{ \theta = (W_1,W_2) \mid \sum_{j = 1}^m \| \weight_{1,j} \|_2 |\weight_{2,j}| \le C\}$.
    Under the setting of~\Cref{thm:sharpness_generalization_complexity}, there exists a absolute constant $C$ independent of $d$ and $\delta$, such that with probability at least $1 - \delta$, $\argmin_{L(\para) = 0}\mathrm{Tr}\left( \nabla^2 L \left( \theta \right) \right) \subseteq \Theta_C$ and $\rada_S(\{  \modelnobias_{\theta} \mid \theta \in \Theta_C\}) \le \tilde O\left(\sqrt{d/n}\right)$.
\end{lemma}

We would like to note that similar results of~\Cref{thm:sharpness_generalization_complexity,lem:norm_bound} have also been shown in a prior work~\cite{wu2023implicit} (that the authors were not aware of before the first version of this work was online).

The almost linear complexity in~\Cref{thm:sharpness_generalization_complexity} is not trivial. For example,~\cite{wei2019regularization} shows that learning the distribution will require $\Omega(d^2)$ samples for Neural Tangent Kernel (NTK)~\citep{jacot2018neural}. In contrast, our result shows that learning the distribution only requires $\tilde O(d)$ samples as long as the flatness of the model is controlled.

Beyond reducing model complexity, flatness may also encourage the model to find a more interpretable solution. We prove that under a stronger than i.i.d condition over the training set, the near flattest interpolating model with architecture \mlpsimplebn~will provably generalize and the weight of the first layer will be centered on the first two coordinates of the input, i.e., $ \| \weight_{1,i}[3:d]\|_2 \le \epsilon \| \weight_{1,i}\|_2$. 

\begin{condition}[Complete Training Set Condition]
    \label{condition:complete_training}
    There exists set $S \subset \{-1,1\}^{d - 2}$, such that the linear space spanned by 
    $S - S = \{s_1 - s_2 \mid s_1, s_2 \in S\}$ has rank $d - 2$ and the training set is $\{(x, y) \mid x \in \R^d, x[3:d] \in S, x[1], x[2] \in \{-1,1\}, y = x[1] \times x[2]\}$.
    \end{condition}

    \begin{theorem}
    \label{thm:sharpness_generalization_feature}
    Given any training set $\{(x_i,y_i)\}_{i=1}^n$satisfying~\Cref{condition:complete_training}, for any width $m$ and any $\epsilon > 0$, 
    there exists constant $\kappa> 0$, such that for any width-$m$ \mlpsimplebn~, $\modelbn$, satisfying $\modelbn_\para$
    interpolates the training set and $\mathrm{Tr}\left( \nabla^2 L(\theta) \right)  \le \kappa+  \inf_{\Loss(\para') = 0}\mathrm{Tr}\left(  \nabla^2 L(\para') \right)$,
     it holds that $\forall x \in \{-1,1\}^d, \big |x[1]x[2] - f_{\para}(x) \big| \le \epsilon$ and that $\forall i \in [m], \| \weight_{1,i}[3:d]\|_2 \le \epsilon \| \weight_{1,i}\|_2$.
    \end{theorem}
    One may notice that in~\Cref{thm:sharpness_generalization_feature} we only consider the approximate minimizer of sharpness.
    This is because the gradient of output with respect to $\weight_1, \bias_1$, despite never being zero, will converge to zero as the norm of $\weight_1, \bias_1$ converges to $\infty$.
    
    \Cref{condition:complete_training} may seem stringent. In practice (\Cref{tab:weight}), we find it not necessary for 1-SAM to find a generalizable solution. 
    We hypothesize that this condition is mainly technical.~\Cref{thm:sharpness_generalization_feature} shows that 
    sharpness minimization may guide the model to find an interpretable and low-rank representation. Similar implicit bias of SAM has also been discussed in~\cite{andriushchenko2023sharpness}
    The proof is deferred to~\Cref{sec:appendix_theory_mlp_bn}

\subsection{SAM empirically finds the flattest model that generalizes}
\label{sec:sam_generalize_empiric}

\begin{figure}
    \centering
    \begin{subfigure}[t]{0.45\textwidth}
        \centering
        \includegraphics[width=\textwidth]{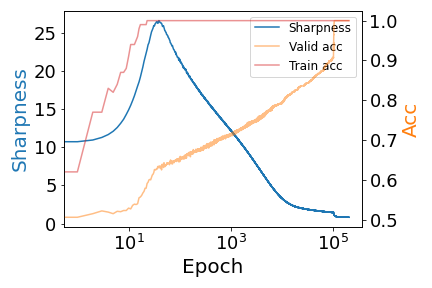}
        \caption{2-layer MLP with simplified BN}
        \label{fig:classification_bn}
    \end{subfigure}
    \begin{subtable}[t]{0.45\textwidth}
    \vspace{-8em}    
    \begin{tabular}[b]{|c|c|c|}\hline
        $W_{1,i}[1]$ & $W_{1,i}[2]$ & $\|W_{1,i}[3:d] \|_2$\\ \hline
        18.581 & -18.582 & 0.02   \\ \hline
        -14.363 & -14.363 & 0.03  \\ \hline
        13.768 & 13.771 & 0.03 \\ \hline
        -12.601 & 12.601 & 0.01 \\ \hline
    \end{tabular}
    \vspace{2em}
    \caption{Weights of the four neurons with the largest norm in the first Layer}
    \label{tab:weight}
\end{subtable}
    \caption{\textbf{Interpretable Flattest Solution} We train a 2-layer MLP with simplified BN using 1-SAM on $\distr$ with dimension $d = 30$ and training set size $n = 100$. 
    After training, we find that the model is indeed interpretable. In~\Cref{tab:weight}, we inspect the weight of the four neurons of 
    the four largest neurons in the first layer and we observe that the four neurons approximately extract features $\pm x[1] \pm x[2]$.}
\end{figure}

We use 1-SAM to train \mlpnobias~on data distribution $\distr$ to verify 
our~\Cref{thm:sharpness_generalization_complexity} (\Cref{fig:classification_no_bias}). As expected, the model interpolates the training set and reaches a flat minimum that generalizes perfectly 
to the test set. 

We then verify our~\Cref{thm:sharpness_generalization_feature} by training a 2-layer MLP with simplified BN on data distribution $\distr$ (\Cref{fig:classification_bn}). 
Here we do not enforce the strong theoretical~\Cref{condition:complete_training}. However, we still observe that SAM finds a flat minimum that generalizes well.
We then perform a detailed analysis of the model and find that the model is indeed interpretable. For example, the four largest neurons in the first layer approximately extract features $\{ \relu(c_1x[1]+c_2 x[2])\mid c_1,c_2\in \{-1,1\}\}$ (\Cref{tab:weight}). Also, the first 2 columns of the weight matrix of the first layer, corresponding 
to the useful features $\{ \relu(c_1x[1]+c_2 x[2])\mid c_1,c_2\in \{-1,1\}\}$, have norms $42.47$ and $42.48$, while the largest column norm of the rest of the weight matrix is only $5.65$.

%% file: Case2.tex
\section{Scenario II: Both Flattest Generalizing and Non-generalizing Models Exist, and SAM Finds the Former}
\label{sec:scenario2}

\subsection{Both generalizing and non-generalizing solutions can be flattest}
In previous section, we show through~\Cref{thm:sharpness_generalization_complexity,thm:sharpness_generalization_feature} 
that sharpness benefits generalization under some assumptions. 
It is natural to ask whether it is possible to extend this bound to general architectures. However, in this section, we 
will show that the generalization benefit depends on model architectures. 
In fact, 
simply adding bias to the first layer of~\mlpnobias~makes non-vacuous generalization bound impossible for distribution $\distr$. This then leads to a negative answer to~\Cref{hypo:flat}.

\begin{definition}[Set of extreme points]
\label{def:extreme_points}
    A finite set $S \subset \R^d$ is a set of extreme points if and only if for any $x \in S$, $x$ is a vertex of the convex hull of $S$. 
\end{definition}

\begin{definition}[Memorizing Solutions]
\label{def:memorizing}
A $D$-layer network is a \emph{memorizing solution} for a training dataset if (1) the network interpolates the training dataset, and (2) for any depth $k \in [D - 1]$, there is an injection from the input data to the neurons on depth $k$, such that the activations in layer $k$ for each input data is a one-hot vector with the non-zero entry being the corresponding neuron.
\end{definition}

\begin{theorem}
    \label{thm:sharpness_counter_example}
    For any $D \ge 2$, if the input data points $\{x_i\}$ of the training set form a set of extreme points (\Cref{def:extreme_points}), 
    then there exists a  width $n$ layer $D$ \deepmlpbias~that is a memorizing solution (\Cref{def:memorizing}) for the training dataset and has minimal sharpness over all the interpolating solutions. 
\end{theorem}

As one may suspect, these memorizing solutions can have poor generalization performance.
\begin{proposition}
    \label{proposition:sharpness_bad_cube}
        For data distribution $\distr$, for any number of samples $n$, there exists a width-$n$~\mlpbias~that memorizes the training 
        set as in~\Cref{thm:sharpness_counter_example}, reaches minimal sharpness over all the interpolating models and has generalization error 
        $\max\{1 - n/2^d, 0\}$ measured by zero one error.
\end{proposition}

This corollary shows that a flat model can generalize poorly. Comparing~\Cref{thm:sharpness_counter_example,thm:sharpness_generalization_complexity}, 
one may observe the perhaps surprising difference caused by slightly modifying the architectures (adding bias or removing the BatchNorm).
To further show the complex relationship between sharpness and generalization, the following theorem suggests, despite the existence of memorizing 
solutions, there also exists a flattest model that \emph{can} generalize well.

\begin{proposition}
\label{thm:sharpness_good_cube}
For data distribution $\distr$, for any number of samples $n$, there exists a width-$n$~\mlpbias~that interpolates the training dataset,
 reaches minimal sharpness over all the interpolating models, and has zero generalization error measured by zero one error. 
\end{proposition}

The flat solution constructed is highly simple. It contains four activated neurons, each corresponding to one feature in $\pm x[1] \pm x[2]$ (\Cref{eq:good_construction}).

\textbf{Proof sketch.} For simplicity, we will consider~\mlpbias~here.
The construction of the memorizing solution in~\Cref{thm:sharpness_counter_example} is as follows (visualized in~\Cref{fig:memorization_illustration}). As the input data points form a set of extreme points (\Cref{def:extreme_points}), for each input data point $x_i$, there exists a vector 
$\|w_i\| = 1, w_i \in \R^d$, such that $\forall j \neq i, w_i^\top x_i > w_i^\top x_j$. We can then choose 
\begin{align*}
    \weight_1 = [ \sqrt{r_i|y_i|} w_i / \epsilon]_i^\top, 
    \bias_1 = [ \sqrt{r_i |y_i|} \left(-w_i^\top x_i + \epsilon\right)/ \epsilon]^\top, 
    \weight_2 = [\sgn(y_i)\sqrt{|y_i| / r_i} ]_i.
\end{align*}
Here $r_i = (\|x_i\|^2 + 1)^{1/2}$ and $\epsilon$ is a sufficiently small positive number. Then it holds that $\ReLU(\weight_1 x_i + \bias_1) = \sqrt{r_i |y_i|} e_i$, where $e_i$ is the $i-$th coordinate vector. This shows there is a one-to-one correspondence between the input data and the neurons. It is easy to verify that the model interpolates the training dataset. Furthermore, for $\distr$ and  sufficiently small $\epsilon$, for any input $x \not \in \{x_i\}_{i \in [n]}$, it holds that $\ReLU(\weight_1 x + \bias_1) = 0$. Hence the model will output the same label $0$ for all the data points outside the training set. This indicates~\Cref{proposition:sharpness_bad_cube}.

\begin{figure}
    \centering
    \begin{minipage}[c]{0.45\textwidth}
        \centering
        \includegraphics[width=1.1\textwidth]{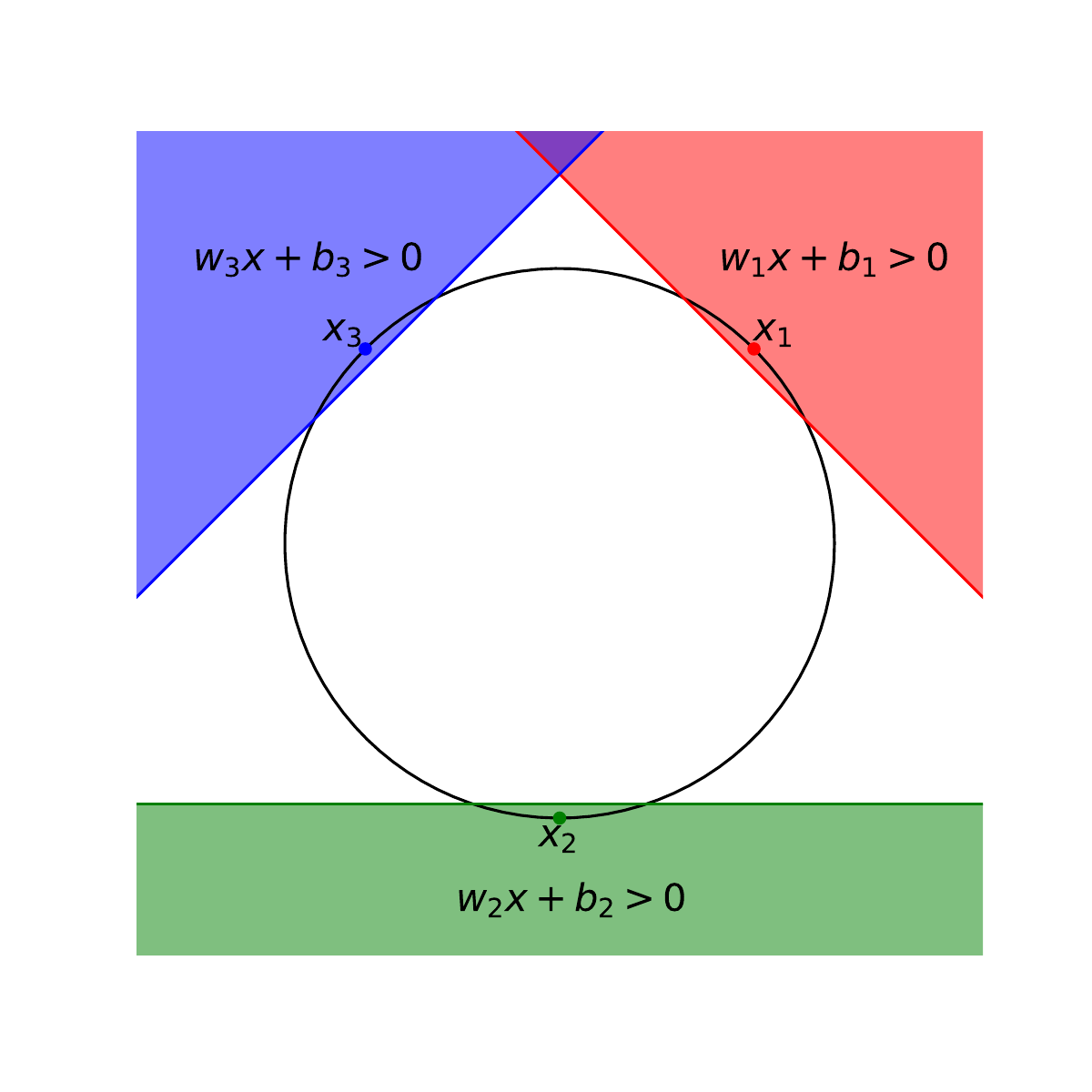}
    \end{minipage}
    \quad
    \begin{minipage}[c]{0.50\textwidth}
       \centering
       \caption{\textbf{Visualization of Memorization Solutions.} This is an illustration of the memorizing solutions constructed in~\Cref{thm:sharpness_counter_example}. Here the input data points come from a unit circle and are marked as dots. The shady area with the corresponding color represents the region where the corresponding neuron is activated. One can see that the network can output the correct label for each input data point in the training set as long as the weight vector on the corresponding neuron is properly chosen. Further, the network will make the same prediction $0$ for all the input data points outside the shady area and this volume can be made almost as large as the support of the training set by choosing $\epsilon$ sufficiently small. Hence the model can interpolate the training set while generalizing poorly.}
       \label{fig:memorization_illustration} 
   \end{minipage}
\vspace{-0.4cm}
\end{figure}

To show the memorization solution has minimal 
sharpness, we need the following lemma that relates the sharpness and the Jacobian of the model.
\begin{lemma}
\label{lemma:sharpness_jacobian}
For mean squared error loss $l_{mse}$, if model $f_\para$ is differentiable and interpolates dataset $\{(x_i, y_i)\}_{i \in [n]}$, then
$\sharpness{\para} =  \frac{2}{n} \sum_{i=1}^n \| \nabla_{\para} \model_\para(\feature_i) \|^2$. 
\vspace{-0.2cm}
\end{lemma}
\begin{proof}[Proof of~\Cref{lemma:sharpness_jacobian}]
By standard calculus, it holds that,
\begin{align}
    \sharpness{\para} &= \frac{1}{n} \sum_{i = 1}^n \Trace\left(\nabla^2_{\theta} \left[(\model_\para(\feature_i) - y_i)^2\right]\right) \notag \\
    &= \frac{2}{n} \sum_{i = 1}^n \Trace\left( \nabla^2_{\para} \model_\para(\feature_i) (\model_\para(\feature_i) - y_i) + \left(\nabla_{\para} \model_\para(\feature_i) \right) \left(\nabla_{\para} \model_\para(\feature_i) \right)^\top \right) \notag \\
    &= \frac{2}{n}\sum_{i = 1}^n \Trace\left( \left(\nabla_{\para} \model_\para(\feature_i) \right) \left(\nabla_{\para} \model_\para(\feature_i) \right)^\top\right) 
    = \frac{2}{n} \sum_{i = 1}^n \|\nabla_{\para} \model_\para(\feature_i) \|_2^2. \label{eq:lemma_jacobian}
\end{align}
The first equation in~\Cref{eq:lemma_jacobian} use $\forall i, f_{\theta}(x_i) = y_i$. The proof is then complete.
\end{proof}
After establishing~\Cref{lemma:sharpness_jacobian}, one can then explicitly calculate the lower bound of $ \| \nabla_{\para} \model_\para(\feature_i) \|^2$ condition on $\model_\para(\feature_i) = y_i$. For simplicity of writing, we will view the bias term as a part of the weight matrix by appending a $1$ to the input data point. Precisely, we will use notation $x_i' \in \R^{d + 1}$ to denote transformed input satisfying $\forall j \in [d], x_i'[j] = x_i[j], x_i'[d + 1] = 1$ and $\weight_1' = [\weight_1, \bias_1] \in \R^{m \times (d + 1)}$ to denote the transformed weight matrix.

By the chain rule, we have,
\begin{align}
    \| \nabla_{\para} \model_\para(\feature_i) \|^2 &= 
    \| \nabla_{\weight_1'} \model_\para(\feature_i) \|_{F}^2  +  \| \nabla_{\weight_2} \model_\para(\feature_i) \|_{F}^2 \notag \\
    &= \| (\weight_2 \odot \one\left[\weight_1' x_i' > 0\right]) x_i'^\top \|_{F}^2 + \| \ReLU\left(\weight_1' x_i'\right)\|_2^2.  \notag \\
    &= \|\weight_2 \odot \one\left[\weight_1' x_i' > 0\right]\|_{2}^2 \|x_i'\|^2 + \| \ReLU\left(\weight_1' x_i'\right)\|_2^2. 
\end{align}
Then by Cauchy-Schwarz inequality, we have
\begin{align}
    \| \nabla_{\para} \model_\para(\feature_i) \|^2 &= 
    \|\weight_2 \odot \one\left[\weight_1' x_i' > 0\right]\|_{2}^2 \|x_i'\|^2 + \| \ReLU\left(\weight_1' x_i'\right)\|_2^2 \notag \\
    &\ge 2 \|x_i'\| \left| \left(\weight_2 \odot \one\left[\weight_1 x_i > 0\right]\right)^\top \ReLU\left(\weight_1' x_i'\right) \right| = 2 \|x_i' \| |y_i|. \label{eq:lower_bound_jacobian}
\end{align}
In~\Cref{eq:lower_bound_jacobian}, we use condition $f_{\theta}(x_i) = y_i$. Finally, notice that the lower bound is reached when 
\begin{align}
\label{eq:optimal_weight}
    \weight_2 \odot \one\left[\weight_1' x_i' > 0\right] = \ReLU\left(\weight_1' x_i'\right)/\|x_i'\|.
\end{align}
Condition~\Cref{eq:optimal_weight} is clearly reached for the memorization construction we constructed, where both sides of the equation are equal to $\sqrt{|y_i|/\|x_i'\|} e_i$. This completes the proof of~\Cref{thm:sharpness_counter_example}.

However, the memorization network is not the only parameter that can reach the lower bound. For example, for distribution $\distr$, if parameter $\theta$ satisfies, 
\begin{align}
    &\forall i, j \in \{0,1\}, \weight_{1,2i + j + 1} = r [(-1)^i, (-1)^j,..., 0],  \bias_1[2i + j + 1] = -r, \weight_2[2i + j] = (-1)^{i + j}/r. \label{eq:good_construction}\\
    &\forall k > 4, \weight_{1, k} = [0,...,0], \bias_1[k] = 0, \weight_2[k] = 0, \notag
\end{align}
with $r = (d^2 + 1)^{1/4}$. then for any $x \in \{-1, 1\}^{d}$, it holds that 
$\ReLU(\weight_1 x + \bias_1) = r e_{ 5/2 - x[1] - x[2]/2}$ and $\model_{\para}(x) = x[1] \times x[2]$. Hence it is possible for~\Cref{eq:good_construction} to hold while the model has perfect generalization performance.

\subsection{SAM empirically finds the non-generalizing solutions}
\begin{figure}
    \centering
    \begin{subfigure}[t]{0.45\textwidth}
        \centering
        \includegraphics[width=\textwidth]{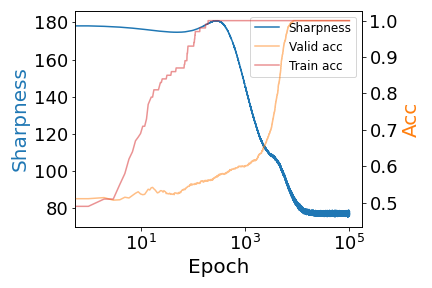}
        \caption{Baseline}
        \label{fig:classification_bias_baseline}
    \end{subfigure}
    \begin{subfigure}[t]{0.45\textwidth}
       \centering
       \includegraphics[width=\textwidth]{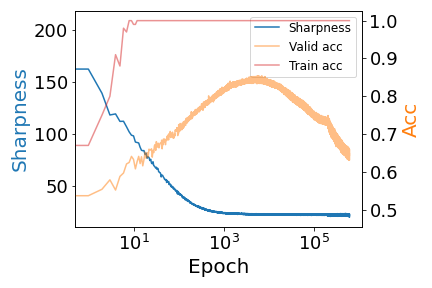}
       \caption{1-SAM}
       \label{fig:classification_bias_sam}
   \end{subfigure}
\caption{
\textbf{Scenario II.} We train a 2-layer MLP with ReLU activation with Bias using gradient descent with weight decay and 1-SAM on $\distr$ with dimension $d = 30$ and training set size $n = 100$. One can clearly observe a distinction between the two settings. The minimum reached by 1-SAM is  flatter but the model fails to generalize and the generalization performance even starts to degenerate after 4000 epochs. The difference between~\Cref{fig:classification_no_bias_sam,fig:classification_bias_sam} indicates a small change in the architecture can lead to a large change in the generalization performance.}
\label{fig:classification_bias}
\end{figure}

\begin{figure}
    \centering
    \begin{subfigure}[t]{0.45\textwidth}
        \centering
        \includegraphics[width=\textwidth]{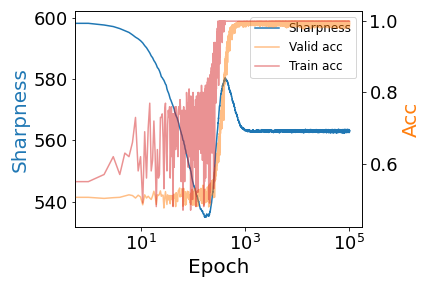}
        \caption{Baseline}
        \label{fig:classification_bias_softplus_baseline}
    \end{subfigure}
    \begin{subfigure}[t]{0.45\textwidth}
       \centering
       \includegraphics[width=\textwidth]{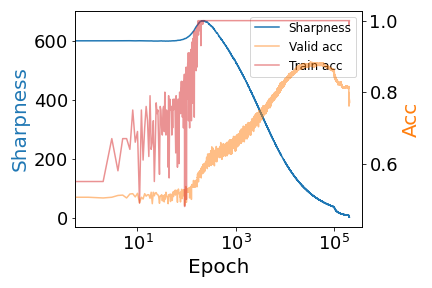}
       \caption{1-SAM}
       \label{fig:classification_bias_softplus_sam}
   \end{subfigure}
\caption{
\textbf{Scenario II with Softplus Activation.} We train a 2-layer MLP with Softplus activation ($\mathsf{SoftPlus}(x) = \log(1 + e^x)$) with bias using gradient descent with weight decay and 1-SAM on $\distr$ with dimension $d = 30$ and training set size $n = 100$. We observe a similar phenomenon as~\Cref{fig:classification_bias}. }
\label{fig:classification_bias_softplus}
\vspace{-0.1in}
\end{figure}

In this section, we will show that in multiple settings, SAM can find solutions that have low 
sharpness but fail to generalize compared to the baseline full batch gradient descent method with weight decay.
This proves that flat minimization can hurt generalization performance. However, one should note that~\Cref{hypo:work} is not denied for the current architectures.

\paragraph{Converged models found by SAM fail to generalize.}
We perform experiments on data distribution $\distr$ in~\Cref{fig:classification_bias}. We apply small learning rate gradient descent with weight decay as our baseline and observe that the converged model found by SAM has a much lower sharpness than the baseline. However, the generalization performance of SAM is much worse than the baseline.
Moreover, 
the generalization performance even starts to degenerate after 4000 epochs.
We conclude that in this scenario, sharpness minimization can empirically hurt generalization performance.

\paragraph{1-SAM may fail to generalize with other activation functions.} A natural question is whether the phenomenon that 1-SAM fails to generalize is limited to 
ReLU activation. In \Cref{fig:classification_bias_softplus}, we show empirically that 1-SAM fails to generalize for 2-layer networks with softplus activation trained on the same dataset, although there is no known guarantee for the existence of memorizing solutions.

%% file: Case3.tex
\section{Scenario III: Both Flattest Generalizing and Non-generalizing Models Exist, and SAM Finds the Latter}
\label{sec:scenario3}
\subsection{Both generalizing and non-generalizing solutions can be flattest}
\label{sec:scenario3_theory}
Despite the surprising contrary between~\Cref{thm:sharpness_generalization_complexity,thm:sharpness_counter_example},
experiments show that~\Cref{hypo:work} consistently hold. However, we will provide a 
counterexample in this section. Specifically, we will consider data distribution $\distr$ and 2-layer ReLU MLP with simplified LayerNorm. 
One can first show both generalizing and non-generalizing solutions exist similar to~\Cref{thm:sharpness_counter_example,proposition:sharpness_bad_cube,thm:sharpness_good_cube}.

\begin{theorem}
\label{thm:sharpness_counter_example_ln}
If the input data points $\{x_i\}$ of the training set form a set of extreme points (\Cref{def:extreme_points}), for sufficiently small $\epsilon$, 
then there exists a  width-$n$ \mlpsimpleln~with hyperparameter $\epsilon$ that is a memorizing solution (\Cref{def:memorizing}) for the training dataset and has minimal sharpness over all the interpolating solutions. 
\end{theorem}

\begin{proposition}
\label{proposition:sharpness_bad_cube_ln}
For data distribution $\distr$, for sufficiently small $\epsilon$,  for any number of samples $n$, there exists a width-$n$ \mlpsimpleln~with hyperparameter $\epsilon$ that memorizes the training 
set as in~\Cref{thm:sharpness_counter_example}, reaches minimal sharpness over all the interpolating models and has generalization error 
$\max\{1 - n/2^d, 0\}$ measured by zero one error.
\end{proposition}

\begin{proposition}
\label{thm:sharpness_good_cube_ln}
For data distribution $\distr$,  for sufficiently small $\epsilon$, for any number of samples $n$, there exists a width-$n$ \mlpsimpleln~with hyperparameter $\epsilon$ that interpolates the training dataset,
 reaches minimal sharpness over all the interpolating models, and has zero generalization error measured by zero one error. 
\end{proposition}

The construction and intuition behind~\Cref{thm:sharpness_counter_example_ln,proposition:sharpness_bad_cube_ln,thm:sharpness_good_cube_ln} are similar to that of~\Cref{thm:sharpness_counter_example,proposition:sharpness_bad_cube,thm:sharpness_good_cube}. The 
proof is deferred to~\Cref{sec:appendix_theory}.
\vspace{-0.05in}
\subsection{SAM empirically finds generalizing models}
\begin{figure}
    \centering
    \begin{subfigure}[t]{0.45\textwidth}
        \centering
        \includegraphics[width=\textwidth]{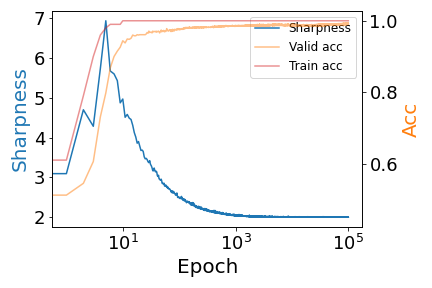}
        \caption{Standard Training}
        \label{fig:classification_simple_ln_bias}
    \end{subfigure}
    \begin{subfigure}[t]
        {0.45\textwidth}
        \centering
        \includegraphics[width=\textwidth]{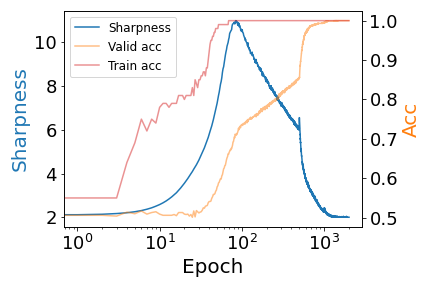}
        \caption{Projected Training}
        \label{fig:classification_simple_ln_bias_norm_bounded}
    \end{subfigure}
    \caption{\textbf{Scenario III.} We train two-layer ReLU networks with simplified LayerNorm on data distribution $\distr$ with dimension $d = 30$ and sample complexity $n = 100$ using 1-SAM. In~\Cref{fig:classification_simple_ln_bias}, we use standard training.
    In~\Cref{fig:classification_simple_ln_bias_norm_bounded}, we restricted the norm of the weight and the bias of the first layer as $10$, to avoid minimizing the sharpness by simply increasing the norm. We can see that in both cases, the models almost perfectly generalize.}
\label{fig:simple_ln_whole}
\vspace{-0.3cm}
\end{figure}

\vspace{-0.05in}
Notice in~\Cref{sec:scenario3_theory} our theory makes the same prediction as in~\Cref{sec:scenario2}. However, strikingly, the experimental observation is reversed (\Cref{fig:simple_ln_whole}). 
Now running SAM can greatly improve the generalization performance till the model perfectly generalizes. 
This directly denies~\Cref{hypo:work} as now we have a scenario in which sharpness minimization algorithms can improve generalization till perfect generalization while there exists a flattest minimizer that will generalize poorly.

%% file: conclusion.tex
\section{Discussion and Conclusion}

We present theoretical and empirical evidence for  (1) whether sharpness minimization implies generalization subtly depends on the choice of architectures and data distributions, and
(2) sharpness minimization algorithms including SAM may still improve generalization even when there exist flattest models that generalize poorly. Our results suggest that low sharpness may not be the only cause of the generalization benefit of sharpness minimization algorithms.

\paragraph{Limitations and future work.} Our results only cover a small subset of existing architectures. A natural extension for our work will be to examine whether flatness implies generalization for deep networks without the bias terms or if the flattest memorizing models exist for such architecture.

In our work, we assume 1-SAM always finds a valid global minimizer for sharpness. However, in previous works, only a local tendency of decreasing sharpness is proven. In all our experiments where the sharpness lower bound can be exactly characterized, we observe that the converged model found by 1-SAM always approximately reaches the lower bound. A possible future direction is characterizing under what condition can 1-SAM be used as a global optimizer for sharpness over minimizers.

Our work also suggests that there does not exist a universal generalization theory for neural networks only based on sharpness. A broader problem would be what other properties of the model can be used to explain the generalization of neural networks.

\section*{ACKNOWLEDGEMENTS}
The authors would like to thank the support from NSF IIS 2045685.

%% file: appendix.tex
\input{appendix_theory}

\newpage
\input{appendix_experiment}

%% file: appendix_theory.tex
\section{Omitted Proofs}
\label{sec:appendix_theory}
We will define $\argmin_{\theta}(f(\theta))$ as the set of the minimizers $f$. When $f$ has a unique minimizer, we also overload the definition to refer to that element. We will use $\mathbf{1}$ to denote the indicator function.
\input{appendix_bn}
\input{appendix_complexity}
\input{appendix_bias}
\input{appendix_ln}
\input{appendix_loss}
\input{appendix_technical}

%% file: appendix_bn.tex
\subsection{Formal results for~\mlpsimplebn}
\label{sec:appendix_theory_mlp_bn}

\subsubsection{Proof of~\Cref{thm:sharpness_generalization_feature}}

We will first prove a lemma showing that the model with a \emph{sparse} second layer weight in L1 norm will satisfy the conclusion of~\Cref{thm:sharpness_generalization_feature}.

\begin{lemma}
\label{lem:sharpness_generalization_feature}
Given any training set$\{(x_i,y_i)\}_{i=1}^n$ satisfying~\Cref{condition:complete_training}, for any width $m$ and any $\epsilon > 0$, for any width-$m$ \mlpsimplebn~,$\modelbn$ and parameter $\theta$ satisfying  $\modelbn_\para$
interpolates the training set and $ \| \gamma \odot \weight_2 \|_1 =   \inf_{\Loss(\para') = 0} \| \gamma' \odot \weight_2' \|_1$, 
then it holds that $\forall x \in \{-1,1\}^d, \big |x[1]x[2] - f_{\para}(x) \big| = 0$ and $\forall i \in [m], \| \weight_{1,i}[3:d]\|_2 = 0$.
\end{lemma}

\begin{proof}[Proof of~\Cref{lem:sharpness_generalization_feature}]
As the training set satisfies~\Cref{condition:complete_training}, there is an integer $r$ such that $n = 4r$ and we assume WLOG  for $ k \in {0,1,..,r - 1}$, $\{x_{4k + t}\}_{t \in \{0,1,2,3\}}$ has the same last $d - 2$ coordinates and the first two coordinates are $(1,1), (1,-1), (-1,1), (-1,-1)$ respectively.
Define $a_i \triangleq \weight_1 x_i + b_1$. Further define $a$ as 
$a[j] = \sqrt{\sum_{i \in [n]}\frac{1}{n} \relu(a_i[j])^2}$.
As the model interpolates the training set, we have 
\begin{align*}
    \sum_{j = 1}^m (\gamma \odot \weight_2)[j]  y_i \relu(a_i[j]) / a[j] = 1, \forall i \in [n].
\end{align*}

Summing the above $n$ equalities, we have that
\begin{align*}
    \sum_{j = 1}^m (\gamma \odot \weight_2)[j] \sum_{i = 1}^n y_i \relu(a_i[j]) / a[j] = n.
\end{align*}

This then implies 
\begin{align*}
    \sum_{j = 1}^m |(\gamma \odot \weight_2)[j]|  \ge \frac{n}{\max_j |\sum_{i = 1}^n y_i \relu(a_i[j]) / a[j]|}.
\end{align*}

However, we have by Cauchy-Schwarz inequality that for each $1\le j \le m$,
\begin{align*}
    |\sum_{i = 1}^n y_i \relu(a_i[j])| &\le \sqrt{r\left(\sum_{k = 0}^{r -1} \left(\sum_{t = 0}^3 y_{4k + t} \relu(a_{4k+t}[j])\right)^2 \right)}.
\end{align*}

Notice that $x_{4k + 1} + x_{4k+4}  = x_{4k + 2} + x_{4k + 3}$, it holds that $a_{4k + 1} + a_{4k+4} = 2\bias_1 = a_{4k + 2} + a_{4k + 3}$ and hence by~\Cref{lem:relu_bn}, it holds that
\begin{align*}
    \sqrt{r\left(\sum_{k = 0}^{r -1} \left(\sum_{t = 0}^3 y_{4k + t} \relu(a_{4k+t}[j])\right)^2 \right)} \le \sqrt{r \sum_{i = 1}^n \relu(a_i[j])^2} = \frac{n}{2} a[j].
\end{align*}

This then implies
\begin{align*}
    \sum_{j = 1}^m |(\gamma \odot \weight_2)[j]|  \ge \frac{n}{n/2} = 2.
\end{align*}

One can then show $2$ is the minimum of $\| \gamma \odot W_2\|_1$ over minimizers by choosing the weights as 
\begin{align}
    &\forall i, j \in \{0,1\}, \weight'_{1,2i + j+1} = [(-1)^i, (-1)^j,..., 0],  \bias'_1[2i + j+1] = -1, \weight'_2[2i + j] = \frac{1}{2}, \gamma[2i + j +1] = 1; \notag \\
    &\forall k > 4, \weight'_{1, k} = [0,...,0], \bias'_1[k] = 0, \weight'_2[k] = 0, \gamma[k] =1. \notag
\end{align}
The training loss is minimized and $\| \gamma \odot W_2\|_1 = 2$.

Hence $\| \gamma \odot \weight_2 \|_1 =   \inf_{\Loss(\para') = 0} \| \gamma' \odot \weight'_2 \|_1$ implies all the inequalities above must be equality. This implies $\forall j \in [m],| \sum_{i = 1}^n y_i \relu(a_i[j])| / a[j] = \frac{n}{2}$, which by~\Cref{lem:relu_bn} then implies the following condition: for any $j \in [m]$, there exists $t_j$ such that $a_{4k + t}[j] \le 0$ for $t \neq t_j$ and $a_{4k + t_j}[j]$ is a constant independent of $k$. \footnote{Notice that for any $k$, the order of $a_{4k + t}[j]$ is the same.}

This implies for any $s_1, s_2 \in S$ with $S$ defined in~\Cref{condition:complete_training}, it holds that $ \forall i \in [m], \weight_{1,i}[3:d]\left(s_1 - s_2\right) = 0$. As the linear space spanned by $S - S$ has rank $d -2$, this implies that $\forall i \in [m], \| \weight_{1,i}[3:d]\|_2 = 0$. As the model predict correctly over the training set and does not use the last $d - 2$ coordinates, we have that $\forall x \in \{-1,1\}^d, \big |x[1]x[2] - f_{\para}(x) \big| = 0$. The proof is then complete.
\end{proof}

We also have an approximate version of the above lemma.

\begin{lemma}
\label{lem:sharpness_generalization_feature_approx} 
Given any training set $\{(x_i, y_i)\}_{i \in [n]}$ satisfying~\Cref{condition:complete_training}, for any $\eps > 0$ and width $m$, there exists $\kappa > 0$, such that for any width-$m$ \mlpsimplebn~$\modelbn$ parameterized by $\theta$ satisfying $\modelbn_{\theta}$
interpolates the training set and $ \| \gamma\|_2^2 + \| \weight_2 \|_2^2  \le \kappa +  \inf_{\Loss(\para) = 0}
\left( \| \gamma\|_2^2 + \| \weight_2 \|_2^2 \right)$,
it holds that $\forall x \in \{-1,1\}^d, \big |x[1]x[2] - f_{\para}(x) \big| \le \epsilon$ and $\forall i \in [m], \| \weight_{1,i}[3:d]\|_2 \le \epsilon$.
\end{lemma}

\begin{proof}[Proof of \Cref{lem:sharpness_generalization_feature_approx}]

Suppose for any $\kappa > 0$, there exists $\theta_{\kappa}$, such that $\sum_{x \in \{-1,1\}^{d}} \big |x[1]x[2] - f_{\theta_{\kappa}}(x) \big| > \epsilon$, $\Loss(\theta_{\kappa}) = 0$  and $\| \gamma_{\kappa}\|_2^2 + \| \weight_{2,\kappa} \|_2^2 \le \kappa  + \inf_{\Loss(\theta) = 0} \| \gamma\|_2^2 + \| \weight_2 \|_2^2$,
We can normalize the first layer of $\theta_{\kappa}$ such that $\| \weight_{\kappa,1}\|_2^2 + \|\bias_{\kappa,1}\|^2 = 1$ without changing the function represented by the network. Then $(\weight_{\kappa,1}, \bias_{\kappa,1}, \weight_{2,\kappa}, \gamma_{\kappa})$ falls in a bounded set. Therefore there exists an accumulation point $\theta^* = (\weight_1^*, \bias_1^*, \weight_{2}^*, \gamma^*)$ of $\{\theta_{1/i}\}_{i \in \mathbf{N}}$, however as $\Loss(\theta)$ and $ \| \gamma\|_2^2 +\| \weight_2 \|_2^2 $ are continuous functions of $\theta$, this implies that $\Loss(\theta^*) = 0$ and $ \| \gamma^*\|_2^2 + \| \weight_2^* \|_2^2  = \inf_{\Loss(\theta) = 0} \left( \| \gamma\|_2^2 + \| \weight_2 \|_2^2 \right)$. 

Notice by AM-GM inequality we have that $\| \gamma\|_2^2 + \| \weight_2 \|_2^2 \ge 2 \| \gamma \odot \weight_2 \|_1$  and equality holds when $\gamma = \weight_2$. We then have $\inf_{\Loss(\theta) = 0} \left( \| \gamma\|_2^2 + \| \weight_2 \|_2^2 \right) = 2\inf_{\Loss(\theta) = 0} \| \gamma \odot \weight_2 \|_1$ and $\gamma^* = \weight_2^*$. Then by~\Cref{lem:sharpness_generalization_feature}, we have that $\sum_{x \in \{-1,1\}^d }\big |x[1]x[2] - f_{\theta^*}(x) \big| = 0$. However $\theta^*$ is an accumulation point of $\theta_{\kappa}$ satisfying $\sum_{x \in \{-1,1\}^d }\big |x[1]x[2] - f_{\theta_{1/i}}(x) \big| > \epsilon$. This then leads to a contradiction.
\end{proof}
We will now lower bound the sharpness of the model,

\begin{lemma}
\label{lem:sharpness_sparse}
For  any parameter $\para$ for architecture~\mlpsimplebn~satisfying that $\Loss(\para) = 0$, it holds that $\sharpness{\theta} \ge \|W_2\|_2^2 + \| \gamma\|_2^2 $ and $\inf_{\Loss(\para) = 0} \sharpness{\theta} = \inf_{\Loss(\para) = 0} \|W_2\|_2^2 + \| \gamma\|_2^2$. 
\end{lemma}

\begin{proof}[Proof of~\Cref{lem:sharpness_sparse}]
Define $a_i$ and $a$ as in proof of~\Cref{lem:sharpness_generalization_feature}. 
By~\Cref{lemma:sharpness_jacobian}, 
\begin{align*}
    \sharpness{\theta} &=\frac{1}{n} \sum_{i = 1}^n \|\nabla_{\theta} \modelbn_{\theta}(x_i) \|_2^2 \\
    &\ge \frac{1}{n}  \sum_{i = 1}^n \|\nabla_{\gamma} \modelbn_{\theta}(x_i) \|_2^2 + \|\nabla_{W_2} \modelbn_{\theta}(x_i) \|_2^2  \\
    &=  \frac{1}{n}  \sum_{j = 1}^m \sum_{i = 1}^n | W_2[j] a_i[j] / a[j] |_2^2 + | \gamma[j] a_i[j] / a[j] |_2^2 \\
    &=  \|W_2\|_2^2 + \| \gamma\|_2^2.  
\end{align*}
This inequality can approximately reach equality when $W_2 = \gamma$ and $\|W_1\|_2$ is sufficiently large.
\end{proof}

Now we are ready to prove~\Cref{thm:sharpness_generalization_feature}.

\begin{proof}[Proof of~\Cref{thm:sharpness_generalization_feature}]
This is a direct consequence of~\Cref{lem:sharpness_generalization_feature_approx} and~\Cref{lem:sharpness_sparse}.
\end{proof}

%% file: appendix_complexity.tex
\subsection{Formal results for~\mlpnobias}

We will prove a more general result that holds for any data distribution satisfying the following condition.

\begin{condition}
\label{symmetry_Subgaussian}
There exists constant $C$ and $\sigma$ independent of $d$ and $m$ such that the data distribution $\datad$ over data points $x$ and label $y$ satisfies that $x$ is  symmetric\footnote{$x$ and $-x$ equal in distribution.} subgaussian random vector (\Cref{def:subgaussian_vector}) with parameter $\sigma$ and variance $I_d$. 
Further, there exists parameter $\theta = (\weight_1, \weight_2)$ for width-$m$ architecture~\mlpnobias~such that $\Pr \left(\modelnobias_{\theta}(x) = y \right) = 1$ and $\sum_{j = 1}^m \| \weight_{1,j} \|_2 |\weight_{2,j}| \le C$.
\end{condition}

\begin{theorem}
\label{thm:sharpness_generalization_complexity_general}
Given any data distribution $\datad$ satisfying~\Cref{symmetry_Subgaussian}, for any $ \delta \in (0,1)$ and input dimension $d$, for $n = \Omega\left(d \log\left(\frac{d}{\delta}\right)\right)$, with probability at least $1 - \delta$ over the random draw of training set $\{(x_i,y_i)\}_{i=1}^n$ from $\datad^{n}$, let $L(\theta)\triangleq \frac{1}{n}\sum_{i=1}^n \mseloss(\modelnobias_\theta(x_i),y_i)$ be the training loss for  width-$m$ \mlpnobias, it holds that for all $\para^*\in  \argmin_{L(\para) = 0}\mathrm{Tr}\left( \nabla^2 L \left( \theta \right) \right)$,
\begin{align}
&\E_{x, y \sim \datad} \left[\mseloss\left(\modelnobias_{\para^*}\left(x\right),y\right) \right ]\le \tilde O\left(d/n\right). \notag 
\end{align}
\end{theorem}

We will also prove a more general result than~\Cref{lem:norm_bound}.
\begin{lemma}
\label{lem:norm_bound_general}
Define $\Theta_C \triangleq  \{ \theta = (W_1,W_2) \mid \sum_{j = 1}^m \| \weight_{1,j} \|_2 |\weight_{2,j}| \le C\}$.
Under the setting of~\Cref{thm:sharpness_generalization_complexity_general}, there exists a absolute constant $C_1$ independent of $d$ and $\delta$, such that with probability at least $1 - \delta$  over the randomness of dataset $\{(x_i,y_i)\}_{i=1}^n$, $\argmin_{L(\para) = 0}\mathrm{Tr}\left( \nabla^2 L \left( \theta \right) \right) \subseteq \Theta_{C_1}$ and $\rada_S(\{  \modelnobias_{\theta} \mid \theta \in \Theta_{C_1}\}) \le \tilde O\left(\sqrt{d/n}\right)$.
\end{lemma}

\subsubsection{Lemmas for uniform convergence}
\label{sec:uniform_convergence}
We will begin with two uniform convergence bounds that will be used in the proof of~\Cref{lem:norm_bound}.

\begin{lemma}
\label{lem:uniform_convergence_indicator}
Given any data distribution $\datad$ satisfying~\Cref{symmetry_Subgaussian}, there exists constant $C_2 > C_1 > 0$ depending on $\sigma$, for any $ \delta \in (0,1)$, input dimension $d$, and number of samples $n = \Omega\left(d \log\left(\frac{d}{\delta}\right)\right)$, with probability at least $1 - \delta$ over the random draw of set $\{(x_i,y_i)\}_{i=1}^n$ from $\datad^{n}$, for any $w \in \R^d$, we have that,
\begin{align}
\label{eq:uniform_convergence_indicator}
C_2 d \ge \frac{1}{n}\sum_{i = 1}^n \| x_i \|_2^2 \one \left(w^\top x_i > 0\right) \ge C_1 d.
\end{align}
\end{lemma}

\begin{lemma}
\label{lem:uniform_convergence_norm}
Given any data distribution $\datad$ satisfying~\Cref{symmetry_Subgaussian}, there exists constant $C_2 > C_1 > 0$ depending on $\sigma$, for any $ \delta \in (0,1)$, input dimension $d$, and number of samples $n = \Omega\left(d \log\left(\frac{d}{\delta}\right)\right)$, with probability at least $1 - \delta$ over the random draw of set $\{(x_i,y_i)\}_{i=1}^n$ from $\datad^{n}$, for any $w \in \R^d, \|w\|_2 = 1$, we have that,
\begin{align}
\label{eq:uniform_convergence_norm}
C_2 \ge \frac{1}{n}\sum_{i = 1}^n | w^\top x_i |_2^2 \one \left(w^\top x_i > 0\right) \ge C_1.
\end{align}
\end{lemma}

We will first prove Lemma~\ref{lem:uniform_convergence_indicator} by a combination of Concentration Inequalities and uniform convergence bound based on Rademacher complexity.
\begin{proof}[Proof of Lemma~\ref{lem:uniform_convergence_indicator}]
We will first prove the upper bound, by~\Cref{lem:norm_concentration}, it holds that
\begin{align*}
    \Pr(\|X_i\|_2^2 \ge 32 \sigma^2 d + 8 \sigma^2 \log(2n / \delta)) \le \frac{\delta}{2n}.
\end{align*}

Hence when $\log(2n / \delta) \le 1024 d$, the proof for the upper bound is complete. If $\log(2n/\delta) > 1024 d$, then by~\Cref{lem:norm_concentration} and Chebyshev's inequality, we have that
\begin{align*}
\Pr(\frac{1}{n} \sum_{i = 1}^n \|x_i\|_2^2 > 3d/2) \le \frac{4\Var(\|x\|_2^2)}{d^2 n} \le \frac{2048 d^2}{n} \le 2048 d^2 \exp(-1024d) \delta \le \delta/2.
\end{align*}
This concludes the proof of the upper bound.

We will then prove the lower bound. We will first show that there exists $\Omega(n)$ data points in $\{x_i\}$, such that $\| x_i \| \ge \Omega(\sqrt{d})$. By~\Cref{lem:lower_bound}, we have there exists constant $\epsilon, \zeta$ such that $\Pr(\|x\|_2^2 > \epsilon d) > \zeta$.

Define indicator $b_i \triangleq \one\left(\|x_i \|_2^2 \ge \epsilon d\right)$, then $b_i$ are i.i.d Bernoulli random variables. We then have $p = \Pr(b_i = 1) > \zeta$.
By Chernoff's bound, it holds that
\begin{align*}
    \Pr\left(\frac{1}{n}\sum_{i = 1}^n b_i \le p/2 \right) \le \exp(-np/8) \le \frac{\delta}{4},
\end{align*}
for any $n > 8\frac{\log(1/\delta)}{\zeta}$. This shows that with probability at least $1 - \frac{\delta}{4}$, we have that,
\begin{align*}
    \frac{1}{n}\sum_{i = 1}^n b_i \ge p/2 \ge \zeta/2.
\end{align*}

This shows that there exists $n' \ge \lfloor \zeta n/2 \rfloor$ data points in $\{x_i\}$, such that $\| x_i \|_2^2 \ge \epsilon d$. We can then relabel the data points as $z_1,...,z_{n'}$. Then $z_i$ are i.i.d random variables with $\|z_i\|_2^2 \ge 1/2$ conditioning on the value of $n'$. We can then have for any $w \in \R^d$, 
\begin{align*}
    \frac{1}{n} \sum_{i = 1}^n \|x_i\|_2^2 \one\left(w^\top x_i > 0\right) 
    &\ge \frac{1}{n} \sum_{i = 1}^{n'} \|z_i\|_2^2 \one\left(w^\top z_i > 0\right)\\
    &\ge \frac{\zeta}{2}
    \frac{1}{n'} \sum_{i = 1}^{n'} \|z_i\|_2^2 \one\left(w^\top z_i > 0\right) \\
    &\ge \frac{\zeta \epsilon d}{2} \frac{1}{n'} \sum_{i = 1}^{n'} \one\left(w^\top z_i > 0\right).
\end{align*}

Finally, define $\gF = \{ z \to \one\left(w^\top z > 0\right) \}$, by~\Cref{lem:vc_dimension}, we have that $\VC(\gF) \le d$. By~\Cref{cor:rademacher_vc}, we have that the empirical Rademacher complexity of $\gF$ on $\{z_i\}_{i \in [n']}$ is upper bounded by $ \sqrt{\frac{4d \log {n'}}{n'}} \le 4\sqrt{\frac{ d \log {n}}{\zeta n}}$.

By~\Cref{lem:uniform_convergence}, with probability at least $1 - \frac{\delta}{4}$, we have that
\begin{align*}
    \sup_{f \in \gF} \left| \frac{1}{n'}\sum_{i = 1}^{n'} f(z_i) - \E[f(z)]\right| \le 2 \gR_{n'}(\gF) + 3\sqrt{\frac{ \log \frac{4}{\delta}}{n'}}.
\end{align*}

The symmetry of $x_i$ implies that $\E[f(z_i)\mid n'] = 1/2$. This shows that with probability at least $1 - \frac{\delta}{4}$, we have that for any $w \in \R^d$,
\begin{align*}
    \frac{1}{n'}\sum_{i = 1}^{n'} \one\left(w^\top z_i > 0\right)  &\ge  \frac{1}{2}- 2 \gR_{n'}(\gF) - 3\sqrt{\frac{ \log \frac{4}{\delta}}{n'}} \\
&\ge \frac{1}{2} - 8\sqrt{\frac{ d \log {n}}{\zeta n}} -3 \sqrt{\frac{ \log \frac{4}{\delta}}{\zeta n}}.
\end{align*}

Hence when $n = \Omega(d\log(d/\delta))$, with probability at least $1 - \delta/2$, we have that for any $w \in \R^d$,
\begin{align*}
    \frac{1}{n} \sum_{i = 1}^n \|x_i\|_2^2 \one\left(w^\top x_i > 0\right) \ge 
    \frac{\epsilon C_d}{2} \frac{1}{n'}\sum_{i = 1}^{n'} \one\left(w^\top z > 0\right) \ge \frac{\epsilon\zeta d}{8}.
\end{align*}

This concludes our proof. 
\end{proof}

We will then prove~\Cref{lem:uniform_convergence_norm}, and we will use the following lemma motivated by~\cite{matouvsek2008variants}.

\begin{lemma}
\label{lem:jl}
Given any data distribution $\datad$ satisfying~\Cref{symmetry_Subgaussian}, there exists $C$ depending on $\sigma$, such that for any $ \delta \in (0,1)$, input dimension $d$, number of samples $n \ge C \log(2/\delta)$, and $w \in \R^d, \|w\|_2 = 1$, with probability at least $1 - \delta$ over the random draw of set $\{(x_i,y_i)\}_{i=1}^n$ from $\datad^{n}$,  we have that,
\begin{align}
\label{eq:jl}
\frac{3}{2} \ge \frac{1}{n}\sum_{i = 1}^n | w^\top x_i |^2 \ge \frac{1}{2}.
\end{align}
\end{lemma}

\begin{proof}[Proof of~\Cref{lem:jl}]
    Notice that $w^\top x_i$ is a subgaussian random variable with parameter $\sigma^2$. Hence by~\Cref{lem:subgaussian_subexponential}, $w^\top x_i$ is a subexponential random variable with expectation $1$. The rest follows from~\Cref{lem:subexponential_concentration}.
\end{proof}

This lemma can be viewed as a variant of the Johnson-Lindenstrauss lemma. We will now proceed to show that we can prove a similar high probability bound when the indicator function $\one\left(w^\top x_i > 0\right)$ is taken into account.

\begin{lemma}
\label{lem:jl_relu}
Given any data distribution $\datad$ satisfying~\Cref{symmetry_Subgaussian}, there exists $C$ depending on $\sigma$, such that for any $ \delta \in (0,1)$, input dimension $d$, sample complexity $n \ge C \log(4/\delta)$ and $w \in \R^d, \|w\|_2 = 1$, with probability at least $1 - \delta$ over the random draw of set $\{(x_i,y_i)\}_{i=1}^n$ from $\datad^{n}$, it holds that,
\begin{align}
\label{eq:jl_relu}
\frac{3}{2} \ge \frac{1}{n}\sum_{i = 1}^n | w^\top x_i |^2 \one\left(w^\top x_i > 0\right) \ge \frac{1}{8}.
\end{align}
\end{lemma}

\begin{proof}[Proof of~\Cref{lem:jl_relu}]
    The upper bound is a direct consequence of~\Cref{lem:jl}.

    We will now prove the lower bound. We will first use a doubling trick using the symmetry of the data distribution. Randomly sample $\{b_i\}_{i \in n}$ uniformly from $\{-1,1\}^n$, and define $z_i = b_i x_i$. We have that $z_i$ and $x_i$ equals in distribution. Hence, we have that $ | w^\top x_i |^2 \one\left(w^\top x_i > 0\right)$ and $ | w^\top b_i x_i |^2 \one\left(b_i w^\top x_i > 0\right)$ equals in distribution. As $b_i$ is independent with $x_i$, this shows that $| w^\top x_i |^2 \one\left(w^\top x_i > 0\right)$ equals in distribution to $| w^\top x_i |^2 c_i$ where $c_i$ is a Rademacher random variable independent of $x_i$.\footnote{For special case where $w^\top x_i = 0$, this still holds as both sides are zero.} Hence, we only need to prove that 
    \begin{align}
        \Pr(\frac{1}{n}\sum_{i = 1}^n | w^\top x_i |^2 c_i < \frac{1}{8}) <  \delta/2.
    \end{align}
    By Chernoff bound, we have that
    \begin{align}
        \Pr(\frac{1}{n}\sum_{i = 1}^n c_i \le \frac{1}{4}) \le \exp(-\frac{n}{16}).
    \end{align}
    Hence when $n > 16 \log(4/\delta)$, we have that $\Pr(\frac{1}{n}\sum_{i = 1}^n c_i \le \frac{1}{4}) < \delta/4$. This then implies that,
    \begin{align}
        \Pr(\frac{1}{n}\sum_{i = 1}^n| w^\top x_i |^2 c_i < \frac{1}{8}) \le \Pr(\frac{1}{n}\sum_{i = 1}^n c_i \le \frac{1}{4}) + 
        \Pr(\frac{1}{n}\sum_{i = 1}^n| w^\top x_i |^2 c_i < \frac{1}{8} \mid \frac{1}{n}\sum_{i = 1}^n c_i \ge \frac{1}{4}) \le \frac{\delta}{2},
    \end{align}
    for the last inequality, we apply~\Cref{lem:jl}. This concludes our proof.
\end{proof}

Notice that~\Cref{lem:jl_relu} is a point-wise bound, we will then use the technique of covering to prove a uniform bound. We will first prove a uniform bound for the case where the indicator function is not taken into account.

\begin{definition}[$\epsilon$-covering]
A set $S \in \R^d$ is an $\epsilon$-covering of a set $S' \in \R^d$, if and only if $\forall s' \in S', \exists s \in S, \| s - s'\|_2 \le \epsilon$.
\end{definition}

\begin{lemma}
    For any $\epsilon > 0$, there exists an $\epsilon$-covering $S$ of the unit sphere in $\R^d$ with cardinality smaller than $\left( \frac{2}{\epsilon} + 1 \right)^d$. 
\end{lemma}

\begin{proof}
    Consider a maximal subset $T$ of the unit sphere in $\R^d$ satisfying that $\forall t \neq t' \in T, \| t - t' \|_2 \ge \epsilon$. As $T$ is maximal, it is an $\epsilon$-covering of the unit sphere.  
    
    Further consider set $T' = \{x \mid \exists t, \| x - t\|_2 \le \frac{\epsilon}{2}\}$, which is the union of $|T|$ disjoint balls with radius $\epsilon/2$. Hence $T'$ has volume $C|T| (\frac{\epsilon}{2})^d $ with $C$ being the volume of the unit ball in $\R^d$. However, $T'$ is contained in a ball with radius $1 + \frac{\epsilon}{2}$ centered at origin. Hence it holds that $C|T| (\frac{\epsilon}{2})^d \le C\left(1 + \frac{\epsilon}{2} \right)^d$. This implies $|T| \le \left( \frac{2}{\epsilon} + 1 \right)^d$.
\end{proof}

\begin{lemma}
    \label{lem:jl_uniform}
Given any data distribution $\datad$ satisfying~\Cref{symmetry_Subgaussian}, there exists $C$ depending on $\sigma$, such that for any $ \delta \in (0,1)$, input dimension $d$ and sample complexity $n\ge Cd \log(\frac{d}{\delta})$, with probability at least $1 - \delta$ over the random draw of set $\{(x_i,y_i)\}_{i=1}^n$ from $\datad^{n}$, it holds that for any $w \in \R^d, \|w\|_2 = 1$, 
\begin{align}
\label{eq:jl_uniform}
2 \ge \frac{1}{n}\sum_{i = 1}^n| w^\top x_i |^2 \ge \frac{1}{4}.
\end{align}
\end{lemma}
\begin{proof}[Proof of~\Cref{lem:jl_uniform}]
Consider a $1/16$ covering of the unit sphere in $\R^d$, $w_1,...,w_N$, we have that $N \le 64^d$. By~\Cref{lem:jl} and Union Bound, we have with probability at least $1 - \delta$ over the random draw of set $\{(x_i,y_i)\}_{i=1}^n$ from $\datad^{n}$, for any $k \in [N]$, we have that,
\begin{align*}
   \frac{3}{2} \ge \frac{1}{n}\sum_{i = 1}^n  | w_k^\top x_i  |^2 \ge \frac{1}{2}.
\end{align*}

Now suppose the above event happens and
\begin{equation}
    w^* = \argmax_{w \in \R^d, \|w\|_2 = 1} \frac{1}{n}\sum_{i = 1}^n| w^\top x_i |^2,
     \gamma = \max_{w \in \R^d, \|w\|_2 = 1} \frac{1}{n}\sum_{i = 1}^n| w^\top x_i |^2.
\end{equation}

Since $\{w_i\}_{i=1}^1$ is a 1/16 covering of unit sphere, there exists $k\in [N]$ such that  $\|w^* - w_k\| \le \frac{1}{16}$, then we have that by Cauchy-Schwarz inequality,

\begin{align*}
     \gamma - \frac{3}{2} &\le \frac{1}{n}\sum_{i = 1}^n | w^{*\top} x_i |^2 - \frac{1}{n}\sum_{i = 1}^n | w_k^\top x_i |^2\\
     &\le \frac{1}{n}\sum_{i = 1}^n | (w^* - w_k)^\top x_i |_2 | (w^* + w_k)^\top x_i |_2 \\
     &\le \sqrt{\frac{1}{n}\sum_{i = 1}^n | (w^* - w_k)^\top x_i |^2} \sqrt{\frac{1}{n}\sum_{i = 1}^n | (w^* + w_k)^\top x_i |^2} \\
     &\le \gamma \|w^* - w_k\| \|w^* + w_k\| \\
     &\le \frac{\gamma}{8}. 
\end{align*}

This then implies that $\gamma \le 2$. Hence, with probability $1 - \delta$, we have that the upper bound holds.

Now for any $w \in \R^d, \|w\|_2 = 1$, suppose $\| w - w_k \|_2 \le \frac{1}{16}$, we have that
\begin{align*}
    \frac{1}{n} \sum_{i = 1}^n | w^{\top} x_i |^2  &\ge
    \frac{1}{n}\sum_{i = 1}^n  | w_k^\top x_i  |^2 + \frac{1}{n}\sum_{i = 1}^n \left(| w^{\top} x_i |^2 -  | w_k^\top x_i  |^2\right) \\
    &\ge \frac{1}{2} - \frac{1}{n} \sum_{i = 1}^n | (w^* - w_k)^\top x_i | | (w^* + w_k)^\top x_i | \\
    &\ge \frac{1}{2} - \sqrt{\frac{1}{n}\sum_{i = 1}^n | (w^* - w_k)^\top x_i |^2} \sqrt{\frac{1}{n}\sum_{i = 1}^n | (w^* + w_k)^\top x_i |^2} \\
    &\ge \frac{1}{2} - \gamma \frac{1}{8} \ge \frac{1}{4}.
\end{align*}
This shows that with probability $1 - \delta$, the lower bound holds as well. The proof is complete.
\end{proof}

We can now prove~\Cref{lem:uniform_convergence_norm}.

\begin{proof}[Proof of~\Cref{lem:uniform_convergence_norm}]
    By~\Cref{lem:jl_uniform}, we have that with probability at least $1 - \delta/2$ over the random draw of set $\{(x_i,y_i)\}_{i=1}^n$ from $\datad^{n}$, for any $w \in \R^d, \|w\|_2 = 1$, we have that,
\begin{align*}
    \frac{1}{n}\sum_{i = 1}^n| w^\top x_i |^2 \in [\frac{1}{4}, 2].
\end{align*}

This directly implies the upper bound. For lower bound, consider a $1/64$ covering of the unit sphere in $\R^d$, $w_1,...,w_N$, we have that $N \le 128^d$. By~\Cref{lem:jl_relu} and Union Bound, we have with probability at least $1 - \delta/2$ over the random draw of set $\{(x_i,y_i)\}_{i=1}^n$ from $\datad^{n}$, for any $k \in [N]$, we have that,
\begin{align*}
    \frac{1}{n}\sum_{i = 1}^n  | w_k^\top x_i  |^2 \one\left(w_k^\top x_i > 0\right) \ge \frac{1}{8}.
\end{align*}

Now suppose the above event happens and for any $w \in \R^d, \|w\|_2 = 1$, suppose $\| w - w_k \| \le \frac{1}{32}$, by~\Cref{lem:relu_lipschitz}, we have that
\begin{align*}
    &\frac{1}{n} \sum_{i = 1}^n | w^{\top} x_i |^2 \one\left(w^\top x_i > 0\right) \\ \ge&
    \frac{1}{n}\sum_{i = 1}^n  | w_k^\top x_i  |^2 \one\left(w_k^\top x_i > 0 \right) + \frac{1}{n}\sum_{i = 1}^n \left(| w^{\top} x_i |^2\one\left(w^\top x_i > 0\right) -  | w_k^\top x_i  |^2 \one\left(w_k^\top x_i > 0\right)\right)
    \\ \ge&
    \frac{1}{8} - \frac{1}{n} \sum_{i = 1}^n | (w - w_k)^\top x_i |_2 ( |w^\top x_i| + |w_k^\top x_i | ) \\
    \\ \ge&
    \frac{1}{8} - \sqrt{\frac{1}{n}\sum_{i = 1}^n | w^\top x_i |^2} \sqrt{\frac{1}{n}\sum_{i = 1}^n | (w - w_k)^\top x_i |^2} -
    \sqrt{\frac{1}{n}\sum_{i = 1}^n | w_k^\top x_i |^2} \sqrt{\frac{1}{n}\sum_{i = 1}^n | (w - w_k)^\top x_i |^2} \\
    \ge& \frac{1}{8} - 2 \times 2 \times \frac{1}{64} = \frac{1}{16}.
\end{align*}
This completes the proof.
\end{proof}

\subsubsection{Proof of~\Cref{lem:norm_bound_general}}

Based on~\Cref{sec:uniform_convergence}, we are now ready to show that for~\mlpnobias, sharpness is within a constant factor of the norm of the parameters.

\begin{lemma}
\label{lem:uniform_convergence_trace}
Given any data distribution $\datad$ satisfying~\Cref{symmetry_Subgaussian}, there exists constant $C_2 > C_1 > 0$ depending on $\sigma$, for any $ \delta \in (0,1)$, input dimension $d$ and number of samples $n = \Omega\left(d \log\left(\frac{d}{\delta}\right)\right)$, with probability at least $1 - \delta$ over the random draw of training set $\{(x_i,y_i)\}_{i=1}^n$ from $\datad^{n}$, for any parameter $\theta = (\weight_1, \weight_2)$ of~\mlpnobias~satisfying that $\Loss(\theta) = 0$, it holds that,
\begin{align*}
    C_2 \left(\|\weight_1\|_F^2 + d \|\weight_2\|^2 \right) \ge \sharpness{\theta} \ge C_1 \left(\|\weight_1\|_F^2 + d \|\weight_2\|^2 \right).
\end{align*}
\end{lemma}

\begin{proof}[Proof of \Cref{lem:uniform_convergence_trace}]
    By~\Cref{lemma:sharpness_jacobian}, we have that,
    \begin{align*}
        \sharpness{\theta} &= \frac{2}{n}\sum_{i = 1}^n \left\| \frac{\partial \Loss}{\partial \weight_1} \right\|_F^2 + \left\| \frac{\partial \Loss}{\partial \weight_2} \right\|_2^2 \\
        &=   \frac{2}{n}\sum_{i = 1}^n \left(\|\weight_2 \odot \one\left[\weight_1 x_i > 0\right]\|_{2}^2 \|x_i\|^2 + \| \ReLU\left(\weight_1 x_i\right)\|_2^2 \right)\notag  \\
        &= \sum_{j = 1}^m  \left( \|\weight_{2,j}\|_2^2 \left(\frac{2}{n}\sum_{i = 1}^n \one\left[\weight_{1,j} x_i > 0\right] \|x_i\|^2\right)  + \sum_{i = 1}^n |\ReLU\left(\weight_{1,j} x_i\right)|^2\right). 
    \end{align*}
By~\Cref{eq:uniform_convergence_indicator,eq:uniform_convergence_norm}, there exists $C_2 > C_1$, such that for any $w \in \R^d$, it holds that
\begin{align*}
    \frac{1}{n}\sum_{i = 1}^n \one\left[w^\top x_i > 0\right] \|x_i\|^2 &\in [C_1d/2, C_2d/2]. \\
    \frac{1}{n}\sum_{i = 1}^n | \ReLU\left(w^\top x_i\right)|^2 &\in [C_1 \|w\|^2/2, C_2 \|w\|^2/2].
\end{align*}
This then implies our result. 
\end{proof}

We can now prove~\Cref{lem:norm_bound_general}.
\begin{proof}[Proof of~\Cref{lem:norm_bound_general}]
By~\Cref{symmetry_Subgaussian}, there exists parameter $\theta = (\weight_1, \weight_2)$, such that $\Loss(\theta) = 0$ and $\sum_{j = 1}^m \| \weight_{1,j} \|_2 |\weight_{2,j}| \le C$. We can properly rescale $\weight_{1,j}$ and $\weight_{2,j}$ such that $\|\weight_1\|_F^2 + d \|\weight_2\|^2 = 2\sqrt{d} \sum_{j = 1}^m \| \weight_{1,j} \|_2 |\weight_{2,j}| \le2 C \sqrt{d}$.

Now by~\Cref{lem:uniform_convergence_trace}, we have that there exists $C_2 > C_1 > 0$, such that for any $\theta^* = (W_1^*, W_2^*) \in \argmin_{\Loss(\theta) = 0} \sharpness{\theta}$, it holds that
\begin{align*}
2C_2 C \sqrt{d} &\ge  C_2 \left(\|\weight_1\|_F^2 + d \|\weight_2\|^2 \right) \\
&\ge \sharpness{\theta} \ge \sharpness{\theta^*} \\
&\ge C_1 \left(\|\weight^*_1\|_F^2 + d \|\weight^*_2\|^2 \right) \\
&= 2C_1 \sqrt{d}  \sum_{j = 1}^m \| \weight_{1,j}^* \|_2 |\weight_{2,j}^*|.
\end{align*}

This then implies that $\sum_{j = 1}^m \| \weight_{1,j}^* \|_2 |\weight_{2,j}^*| \le \frac{C_2C}{C_1}$, completing the proof of the first claim.

By~\Cref{lem:norm_concentration}, with probability at least $1 - \delta$, we have that $\max_i \| x_i \|_2^2 = \tilde O(\sqrt{d})$.
The second claim then follows from~\Cref{lem:radamacher_complexity_network}.
\end{proof}

\subsubsection{Proof of~\Cref{thm:sharpness_generalization_complexity_general}}

We are now ready to prove~\Cref{thm:sharpness_generalization_complexity_general} based on~\Cref{lem:norm_bound_general}.

\begin{proof}[Proof of~\Cref{thm:sharpness_generalization_complexity_general}]
Based on~\Cref{lem:norm_bound_general}, there exists constant $C_1 > C$ with $C$ defined in~\Cref{symmetry_Subgaussian}, with probability at least $1-\delta$, such that  $\argmin_{\Loss(\theta) = 0} \sharpness{\theta} \subset \Theta_{C_1}$ and $\rada_S(\Theta_{C_1}) =  \tilde O(\sqrt{\frac{d}{n}})$. To get the faster rate $\tilde O(d/n)$, we would like to apply \Cref{thm:complexity_loss}. 
The main technical difficulty to apply~\Cref{thm:complexity_loss} here is that for distribution $\datad$, the loss function $\Loss$ is not necessarily bounded. To address this issue, we will consider a truncated version of the mean squared error (as in~\cite{gatmiry2023inductive}).
\begin{equation}
\label{eq:truncated_loss}
l_c(x, y)=\ell_c(x-y)= \begin{cases}(x-y)^2, & \text { if } x-y \in[-c, c], \\ -(x-y)^2+4 c|x-y|-2 c^2, & \text { if } x-y \in[-2 c,-c] \cup[c, 2 c], \\ 2 c^2, & \text { if } x-y \in(-\infty,-2 c] \cup[2 c, \infty).\end{cases}
\end{equation}

We will choose $c = \tilde O(\sqrt{d})$ as in~\Cref{lem:subgaussian_truncated_second_moment} such that $ \E_{x, y \sim \datad}[\|x\|^2 \one\left[C_1 \|x\| \ge c\right]] = \tilde O(\frac{d}{n})$ and $ \E_{x, y \sim \datad}[\|x\|^2 \one\left[C \|x\| \ge c\right]] = \tilde O(\frac{d}{n})$.
By~\Cref{symmetry_Subgaussian}, we have for $x, y \sim \datad$, there exists $\theta_1^* \in \Theta_{C}$ such that $\modelnobias_{\theta_1^*}(x) = y$, then 
\begin{align*}
    &\E_{x,y \sim \datad}[\mseloss(\modelnobias_{\theta}(x), y)] - \E_{x,y \sim \datad}[l_c(\modelnobias_{\theta}(x), y)] \\
    \le& \E_{x, y \sim \datad}[(\modelnobias_{\theta}(x) - y)^2 \one\left[|\modelnobias_{\theta}(x) - y| \ge c\right]] \\
    \le& 2\E_{x, y \sim \datad}[\modelnobias_{\theta}(x)^2 \one\left[|\modelnobias_{\theta}(x)| \ge c\right]] + 2\E_{x, y \sim \datad}[\modelnobias_{\theta*}(x)^2 \one\left[|\modelnobias_{\theta*}(x)| \ge c\right]].
\end{align*}

As we have $\theta \in \Theta_{C_1}$, it holds that
\begin{align*}
    |\modelnobias_{\theta}(x)| &\le 
    \sum_{i = 1}^m |\weight_{2,i}|\|\weight_{1,i}\|_2 \| x\| \le  C_1 \|x\|.
\end{align*}

This then implies that,
\begin{align*}
    \E_{x, y \sim \datad}[\modelnobias_{\theta}(x)^2 \one\left[|\modelnobias_{\theta}(x)| \ge c\right]] \le 
    C_1^2 \E_{x, y \sim \datad}[\|x\|^2 \one\left[C_1 \|x\| \ge c\right]] = \tilde O(\frac{d}{n}).
\end{align*}

Similarly, $\E_{x, y \sim \datad}[\modelnobias_{\theta^*}(x)^2 \one\left[|\modelnobias_{\theta^*}(x)| \ge c\right]] = \tilde O(\frac{d}{n})$. Hence, we have that
\begin{align*}
    \E_{x,y \sim \datad}[\mseloss(\modelnobias_{\theta}(x), y)] - \E_{x,y \sim \datad}[l_c(\modelnobias_{\theta}(x), y)] = \tilde O(\frac{d}{n}).
\end{align*}

We then define the truncated version of $\Loss$ as $\Loss_c(\theta) = \frac{1}{n}\sum_{i = 1}^n l_c(\weight_2^\top \ReLU(\weight_1 x_i), y_i)$. Then we clearly have $\Loss(\theta) = 0 \implies \Loss_c(\theta) = 0$ Now by~\Cref{thm:complexity_loss}, we have that for any $\theta \in \Theta_{C_1}$ and $\Loss(\theta) = 0$, it holds that with probability at least $1 - \delta/2$,
\begin{align*}
    \E_{x,y \sim \datad}[l_c(\modelnobias_{\theta}(x), y)] \le \tilde O(\frac{d + c^2 \log(1/\delta)}{n}) = \tilde O(\frac{d}{n}).
\end{align*}

This completes the proof.
\end{proof}

\subsubsection{Proof of~\Cref{thm:sharpness_generalization_complexity,lem:norm_bound}}

One can easily construct width $4$~\mlpnobias~such that for $\Pr_{x,y \sim \datad}\left(\modelnobias_{\theta}(x) = y\right) = 1$. For example, one can have that
\begin{align*}
    \weight_1 = \begin{bmatrix}
    1 + \epsilon & 1 - \epsilon & 0 & \cdots \\
    1 + \epsilon & -1 + \epsilon & 0 & \cdots \\
    -1 - \epsilon & 1 - \epsilon & 0 & \cdots \\
    -1 - \epsilon & -1 + \epsilon & 0 & \cdots
    \end{bmatrix}, \weight_2 = \frac{1}{2 - 2\epsilon}\begin{bmatrix}
    1 & -1 & -1 & 1
    \end{bmatrix}.
\end{align*}

Hence $\distr$ satisfies the condition in~\Cref{symmetry_Subgaussian} and this completes the proof of~\Cref{thm:sharpness_generalization_complexity,lem:norm_bound}.

%% file: appendix_bias.tex
\subsection{Formal results For~\deepmlpbias}

We will prove~\Cref{thm:sharpness_counter_example} and a generalization of~\Cref{thm:sharpness_good_cube} in this section. We note that~\Cref{proposition:sharpness_bad_cube} is already proved in~\Cref{sec:scenario2}.

\subsubsection{Proof of~\Cref{thm:sharpness_counter_example}}

We have demonstrated the proof for~\mlpbias~in~\Cref{sec:scenario2}, and the proof for layer-$D$~\deepmlpbias is conceptually similar.

\begin{proof}[Proof of~\Cref{thm:sharpness_counter_example}]

We will still use notation $x_i' \in R^{d + 1}$ to denote transformed input satisfying $\forall j \in [d], x_i'[j] = x_i[j], x_i'[d + 1] = 1$ and $\weight_1' = [\weight_1, \bias_1] \in \R^{m \times (d + 1)}$ to denote the transformed weight matrix.

For the simplicity of writing, we will use the following notations,
\begin{align*}
    a_{i, 0} = x_i',
    a_{i, 1} = \ReLU(W_1 x_i + b_1),
    a_{i,d} = \ReLU(W_d a_{i, d - 1}), d > 1
\end{align*}

We will also use $A_{i,d}$ to denote the diagonal matrix with $\one\left(a_{i,d} > 0 \right)$ as the diagonal entries.

By~\Cref{lemma:sharpness_jacobian} and the chain rule, we have that
\begin{align*}
    \| \nabla_{\para} \modeldeepbias_{\para}(x_i) \|_2^2 &= 
    \sum_{j = 2}^D  \| \nabla_{\weight_j} \model_\para(\feature_i) \|_{F}^2 +  \| \nabla_{\weight_1'} \model_\para(\feature_i) \|_{F}^2  \\
    &= \sum_{j = 1}^{D - 1} \|  W_D A_{i, D - 1} \cdots W_{j + 1} A_{i,j} \|_2^2 \| a_{i, j - 1} \|_2^2 + \| a_{i,D-1}\|_2^2 
\end{align*}

By AM-GM inequality and Cauchy-Schwarz inequality, we have that
\begin{align*}
    \| \nabla_{\para} \modeldeepbias_{\para}(x_i) \|_2^2 
    &= \sum_{j = 1}^{D - 1} \|  W_D A_{i, D - 1} \cdots W_{j + 1} A_{i,j} \|_2^2 \| a_{i, j - 1} \|_2^2 + \| a_{i,D-1}\|_2^2 \\
    &\ge D \left( \left(\Pi_{j = 1}^{D - 1} \|  W_D A_{i, D - 1} \cdots W_{j + 1} A_{i,j} \|_2^2 \| a_{i, j - 1} \|_2^2\right) \| a_{i,D-1}\|_2^2  \right)^{\frac{1}{D}}  \\
    &\ge D \left( \left(\Pi_{j = 1}^{D - 1}\|  W_D A_{i, D - 1} \cdots W_{j + 1} A_{i,j} \|_2^2 \| a_{i, j} \|_2^2  \right) \|x_i'\|_2^2 \right)^{\frac{1}{D}} \\
    &\ge D |y_i|^{2(D-1)/D} \|x_i'\|_2^{2/D}. 
\end{align*}

As every training data point is an extreme point of the convex hull of $\{x_i\}$, for each input data point $x_i$, there exists a vector 
$\|w_i\| = 1, w_i \in \R^d$, such that $\forall j \not \in i, w_i^\top x_i > w_i^\top x_j$. Finally, the above inequality can be reached by a memorizing solution when we choose,
\begin{align*}
    &\weight_1 = [ u_i w_i / \epsilon]_i^\top, 
    \bias_1 = [ u_i  \left(-w_i^\top x_i + \epsilon\right)/ \epsilon]_i^\top, \\
    &\weight_{j} = \mathrm{diag}([1/r_i]_{i \in [n]}), \forall 2 \le j \le D - 1, \\
    &\weight_{D} = [\sign(y_i) / r_i]_{i \in [n]},
\end{align*}
with $r_i, u_i$ satisfyng $r_i = (\| x_i' \| / |y_i|)^{1/D}, u_i = |y_i| r_i^{D - 1}$ when $y_i \neq 0$, $r_i = u_i = 1$ when $y_i = 0$. The proof is then completed.
\end{proof}

\subsubsection{Generalization of~\Cref{thm:sharpness_good_cube}}

We will directly prove a more general version of \Cref{thm:sharpness_good_cube}, which is \Cref{prop:sharpness_good_polygon}.

\begin{proposition}\label{prop:sharpness_good_polygon}
    Given any constant $s$, for any data distribution $\datad$ over input $x$ and label $y$ satisfying that (1) the label $y$ depends only on the first $s$ coordinates $\mathcal{I}$ of the input, (2) $x_{\mathcal{I}}$ are sampled from a set of extreme points in $\R^{|\mathcal{I}|}$ and (3) $\Pr(\|x\|_2 = R) = 1$, for sufficiently large width $m$ depending on $\datad$, there exists a flattest minimizer $\para^*$ for width-$m$~\mlpbias~with generalization error $0$.
\end{proposition}

\begin{proof}[Proof of \Cref{prop:sharpness_good_polygon}]
    The proof is similar to the proof of~\Cref{thm:sharpness_good_cube}. Suppose the set of extreme points in $\R^{|\mathcal{I}|}$ contains $k$ elements $v_1,...,v_k$ satisfying $\|v_k\| = v$ and corresponds to label $y_1,...,y_k$. Then there exist vectors
    $\|w_i\| = 1, w_i \in \R^{|\mathcal{I}|}$, such that $\forall j \neq i, w_i^\top v_i > w_i^\top v_j$. We will then choose $m = k$ and let, 
    \begin{align}
        &\forall j \in [k], \weight_{1,j} = r [v_j,..., 0]/\epsilon,  \bias_1[j] = r(- w_j^\top v_j + \epsilon)/\epsilon, \weight_2[j] = y_j/r, \label{eq:good_construction_general}
    \end{align}
    with $r^2 = |y_j| (R^2 + 1)$ and $\epsilon$ sufficiently small. It is easy to verify that the construction will reach the smallest sharpness for any training set.
\end{proof}

The construction above critically relies on the fact that there exists a set of extreme points in $\R^{|\mathcal{I}|}$ containing the input data points. We will show that this is not necessary by the following example.

\begin{proposition}
    Given any constant $L$, for any data distribution $\datad$ over input $x$ and label $y = f(x)$ satisfying that (1) the label function $f$ depends only on the first $2$ coordinates $\mathcal{I}$ of the input and is $L$-lipschiz, (2) the input data points satisfy $x_{\mathcal{I}}$ are sampled uniformly from the unit circle in $\R^{2}$, and  (3) $\Pr(\|x\|_2 = R) = 1$, for any $\delta \in (0, 1/20)$ and $n = \Omega(\log(1/\delta)/\delta)$, with probability $1 - \delta$ over the random draw of training set $\{(x_i, y_i)\}_{i \in [n]}$,  there exists a flattest minimizer $\para^*$ for width-$n$~\mlpbias~with generalization error $O(\delta^2)$.
\end{proposition}

\begin{proof}

Suppose the largest value of label $y$ is $Y$.
Suppose for dataset $\{(x_i, y_i)\}$, the first two coordinates of $\{x_i\}$ forms a set $\{v_i\}$ that lies on 
the unit circle in $\R^2$ and corresponds to label $\{y_i\}$. Suppose WLOG $v_i$ is sorted by the angle it forms with the $x$-axis. We will then define $z_i$ as the midpoint of the arc $v_{i - 1} v_{i}$ and $w_i$ as the unit vector perpendicular to $z_{i}z_{i + 1}$. Here $z_{n + 1} = z_1$ and $w_0 = w_n$. The flattest minimizers $\para^*$ is then defined as,
\begin{align}
    &\forall j \in [n], \weight_{1,j} = r [w_j,..., 0]/w_j^\top z_j,  \bias_1[j] = r(-w_j^\top v_j + w_j^\top z_j)/w_j^\top z_j, \weight_2[j] = y_j/r, \label{eq:good_construction_general_2}
\end{align}
with $r^2 = |y_j| \sqrt {R^2 + 1}$. Verifying that the construction will reach the smallest sharpness for the training set is easy. Now splitting the sphere into $N = \lceil 2 \pi / \delta \rceil > 1/\delta$ arcs with length no longer than $\delta$. Then by the standard coupon collector problem, with probability at least $1 - \delta$, when $n \ge N \log \delta$, there is at least one point in each arc. Under such case, we have that $z_j z_{j + 1}$ has length no greater than $2 \delta$ and $w_j^\top z_j  > 1 - 10 \delta$ for any $j$.

Therefore, for any $v \in \R^2, \|v\| = 1$, suppose WLOG $v$ fails in arc $z_1z_2$ and corresponds to label $y$, then $\modelbias_{\para^*}(x) = y_1 w_1^\top (v - v_1 + z_1)/w_1^\top z_1$ for $x[1:2] = v$ . Therefore, we have that
\begin{align*}
    \| \modelbias_{\para^*}(x) - y \|_2^2 &\le 
    \| \modelbias_{\para^*}(v) - y_1 \|_2^2 + 
    \| y_1 - y\|_2^2   \\
    &\le \| y_1 w_1^\top (v - v_1)/w_1^\top z_1 \|_2^2  + L^2 \| v - v_1 \|_2^2 \\
    &\le 4 Y^2\delta^2 / (1 - 10\delta)^2 + L^2 \delta^2.
\end{align*}
This shows that the expected generalization error is bounded by $O(\delta^2)$. The proof is completed.
\end{proof}

%% file: appendix_ln.tex
\subsection{Formal results for~\mlpsimpleln}

\subsubsection{Proof of~\Cref{thm:sharpness_counter_example_ln}}

We will first lower bound the sharpness of all minimizers of~\mlpsimpleln~by the following lemma.

\begin{lemma}
\label{lem:lower_bound_sharpness_ln}
Given any number of samples $n$ and $\epsilon > 0$, for any training set $\{(x_i, y_i)\}_{i \in [n]}$ satisfying that the input data points $\{x_i\}$ of the training set form a set of extreme points, for width-$n$ \mlpsimpleln~with hyperparameter $\epsilon$, it holds that 
$$\inf_{\Loss(\theta) = 0}\sharpness{\theta} \geq \frac{2}{n} \sum_{i = 1}^n \min(1, \frac{2}{\epsilon} \sqrt{\|x_i\|_2^2 + 1} |y_i|).$$
\end{lemma}

\begin{proof}
    By~\Cref{lemma:sharpness_jacobian}, we have that
    \begin{align*}
        \sharpness{\theta} &= \frac{2}{n}\sum_{i = 1}^n \| \nabla_{\theta} f_{\theta}(x_i) \|_2^2.
    \end{align*}
    We will then discuss by cases to show the lower bound of $\| \nabla_{\theta} f_{\theta}(x_i) \|_2^2$ for each $i \in [n]$ when $f_{\theta}(x_i) = y_i$, we will continue to use notation $x_i' \in R^{d + 1}$ to denote transformed input satisfying $\forall j \in [d], x_i'[j] = x_i[j], x_i'[d + 1] = 1$ and $\weight_1' = [\weight_1, \bias_1] \in \R^{m \times (d + 1)}$ to denote the transformed weight matrix.
    \begin{enumerate}
        \item If $\|\relu(W_1 x_i + b_i) \|_2 > \epsilon$, then it holds that
        \begin{align*}
            \| \nabla_{\theta} f_{\theta}(x_i) \|_2^2 &\ge 
            \| \nabla_{W_2} f_{\theta}(x_i) \|_2^2 \\
            &= \|\frac{\relu(W_1 x_i + b_i) }{\|\relu(W_1 x_i + b_i) \|_2} \|_2^2 = 1.
        \end{align*}
        \item If $\|\relu(W_1 x_i + b_i) \|_2 \le \epsilon$, then it holds that
        \begin{align*}
            \| \nabla_{\theta} f_{\theta}(x_i) \|_2^2 &\ge 
            \| \nabla_{W_1'} f_{\theta}(x_i) \|_2^2 + \| \nabla_{W_2} f_{\theta}(x_i) \|_2^2 \\
            &= \frac{1}{\epsilon^2} (\|W_2^\top \one\left(\relu(W_1 x_i + b_i) > 0\right) \|_2^2 \| x_i' \|_2^2 + \| \relu(W_1 x_i + b_i) \|_2^2) \\
            &\ge \frac{2}{\epsilon^2} \|x_i'\|_2 |W_2^\top \relu(W_1 x_i + b_i) | \\
            &\ge \frac{2}{\epsilon} \|x_i'\|_2 |y_i|.
        \end{align*}
    \end{enumerate}
    This concludes the proof.
\end{proof}

\begin{proof}[Proof of~\Cref{thm:sharpness_counter_example_ln}]
By~\Cref{lem:lower_bound_sharpness_ln}, we only need to construct a memorizing solution that has sharpness $ \frac{2}{n}\sum_{i = 1}^n \min(1, \frac{2}{\epsilon} \sqrt{\|x_i\|_2^2 + 1} |y_i|)$.

As the input data points form a set of extreme points, for each input data point $x_i$, there exists a vector 
$\|w_i\| = 1, w_i \in \R^d$, such that $\forall j \not \in i, w_i^\top x_i > w_i^\top x_j$.  We can then construct the minimal sharpness solution by choosing for sufficiently small $\delta$,
\begin{align*}
    &\weight_1 = [ u_i w_i / \delta]_i^\top, 
    \bias_1 = [ u_i  \left(-w_i^\top x_i + \delta\right)/ \delta]_i^\top, \weight_{2} = [r_i y_i]_{i \in [n]},
\end{align*}
with $r_i, u_i$ satisfying 
\begin{enumerate}
    \item $r_i = 1, u_i = 2\epsilon $ when $\sqrt{\|x_i\|_2^2 + 1} |y_i| > \epsilon$.
    \item  $r_i = (\frac{\epsilon}{\sqrt{\|x_i\|^2 + 1}|y_i|})^{1/2},  u_i = \epsilon (\frac{\sqrt{\|x_i\|^2 + 1}|y_i|}{\epsilon})^{1/2}$ when $0 < \sqrt{\|x_i\|_2^2 + 1} |y_i| \le \epsilon$.
    \item  $r_i = 0, u_i = 2\epsilon$ when $y_i = 0$.
\end{enumerate}
It is easy to check this is a memorizing solution that minimizes sharpness.\footnote{When $\sqrt{\|x_i\|_2^2 + 1} |y_i| > \epsilon$, one can notice that $\nabla_{W_1'}\model_{\theta}(x_i) = 0$ as the activation in layer $1$ is nonzero only in one dimension.} The proof is then completed.
\end{proof}

\subsubsection{Proof of~\Cref{proposition:sharpness_bad_cube_ln,thm:sharpness_good_cube_ln}}

\begin{proof}[Proof of~\Cref{proposition:sharpness_bad_cube_ln}]
We will suppose $\epsilon < \sqrt{d + 1}$, then for $\distr$, the minimal sharpness is always $2$ by~\Cref{lem:lower_bound_sharpness_ln}. Consider the following construction for sufficiently small $\delta$, 
\begin{align*}
    &\weight_1 = [ 2\epsilon x_i / \delta]_i^\top, 
    \bias_1 = [ 2\epsilon  \left(-d + \delta\right)/ \delta]_i^\top, \weight_{2} = [y_i]_{i \in [n]},
\end{align*}

Then first this is a memorizing solution that minimizes sharpness. Second, the generalization error is $1 - n/2^d$ because for any $x \not \in \{x_i\}_{i \in [n]}$, it holds that
$\relu(W_1 x + b_1) = 0$ and hence $\model_{\para}(x) = 0$, The proof is then completed.
\end{proof}

\begin{proof}[Proof of~\Cref{thm:sharpness_good_cube_ln}]
We will suppose $\epsilon < \sqrt{d + 1}$, then for $\distr$, the minimal sharpness is always $2$ by~\Cref{lem:lower_bound_sharpness_ln}. Consider the following construction for sufficiently small $\delta$, 
\begin{align}
    &\forall i, j \in [2], \weight_{1,2i + j} = 2 \epsilon [(-1)^i, (-1)^j,..., 0],  \bias_1[2i + j] = -2\epsilon, \weight_2[2i + j] = (-1)^{i + j}. \label{eq:good_construction_ln}\\
    &\forall k > 4, \weight_{1, k} = [0,...,0], \bias_1[k] = 0, \weight_2[k] = 0, \notag
\end{align}
This is an interpolating parameter that minimizes sharpness that can perfectly generalize.
\end{proof}

%% file: appendix_loss.tex
\subsection{Discussion on the choice of loss function}
\label{sec:appendix_loss}

In this section, we will show why our theoretical results hold for logistic loss with label smoothing by showing that using the logistic loss with label smoothing yields the same set of minimizers and flattest minimizers as a corresponding problem using mean squared error.

\begin{definition}[Logistic Loss with Label Smoothing]
\label{def:logistic}
Logistic loss with label smoothing probability $p$ is defined as,
$\loss: \logistic(a, b) = - p \log\left(\frac{e^{ba}}{1 + e^a}\right) - (1-p) \log\left(\frac{e^{(1 - b)a}}{1 + e^a}\right), b \in \{0,1\}$. We will denote the training loss yield as $\logistic$ as $\logLoss$.
\end{definition}

\begin{theorem}
\label{thm:loss_minimizer}
For any probability $p \in (0,1)$, and for any training set $\{(x_i, y_i)\}_{i \in [n]}$ satisfying that $x_i \in \R^d$ and $y_i \in \{0,1\}$, let $\gamma_p = \ln(\frac{1-p}{p})$, if the minimum of the mean squared error $\mseLoss$ over set $\{x_i, \gamma_p (2y_i - 1)\}$ is $0$, then the minimizers of $\mseLoss$ over set $\{x_i,  \gamma_p (2y_i - 1)\}$ and the minimizers of $\logLoss$ over set $\{(x_i, y_i)\}$ are the same.
\end{theorem}
\begin{proof}
\vspace{-0.1in}
    This theorem is a direct consequence of the following inequality,
    \begin{align*}
        \logistic(a, b) &= - p \log\left(\frac{e^{ba}}{1 + e^a}\right) - (1-p) \log\left(\frac{e^{(1 - b)a}}{1 + e^a}\right) \ge -p \log p - (1 - p)\log(1-p).
    \end{align*}
    The minimal is reached when $a = (2b - 1)\gamma_p$ where $\gamma_p = \ln(\frac{1-p}{p})$.
\end{proof}

\begin{lemma}
\label{lem:sharpness_jacobian_logistic}
For any probability $p \in (0,1)$, and for any training set $\{(x_i, y_i)\}_{i \in [n]}$ satisfying that $x_i \in \R^d$ and $y_i \in \{0,1\}$, let $\gamma_p = \ln(\frac{1-p}{p})$,
for any model $f_\para$ that is differentiable and interpolates dataset $\{x_i, \gamma_p (2y_i - 1)\}_{i \in [n]}$, it holds that
$\Trace\left(\nabla^2 \logLoss({\para})\right) = \frac{1}{p(1-p)} \frac{1}{n} \sum_{i=1}^n \| \nabla_{\para} \model_\para(\feature_i) \|^2$.
\end{lemma}

\begin{proof}
    By standard calculus, it holds that,
    \begin{align}
        \sharpness{\para} =& \frac{1}{n} \sum_{i = 1}^n \Trace\left(\nabla^2_{\theta} \left[\logistic(\model_{\para}(x_i), y_i) \right]\right) \notag \\
        =& \frac{1}{n} \sum_{i = 1}^n \Trace\left( \partial_{\theta}\left[ \frac{d \logistic(\model_{\para}(x_i), y_i)}{d \model_{\para}(x_i)} \nabla_{\theta} \model_{\para}(x_i)\right]\right) \notag \\
        =& \frac{1}{n} \sum_{i = 1}^n  \frac{d \logistic(\model_{\para}(x_i), y_i)}{d \model_{\para}(x_i)}\Trace\left(\nabla^2_{\theta} \model_{\para}(x_i) \right) \notag \\
        &+ \frac{1}{n} \sum_{i = 1}^n  \frac{d^2 \logistic(a, y_i)}{d a^2} \mid_{a = \model_{\para}(x_i) }\Trace\left(\nabla^2_{\theta} \model_{\para}(x_i) \right) \Trace\left( \left(\nabla_{\para} \model_\para(\feature_i) \right) \left(\nabla_{\para} \model_\para(\feature_i) \right)^\top\right) \notag \\
        =& \frac{1}{n}\sum_{i = 1}^n  \frac{d^2 \logistic(a, y_i)}{d a^2} \mid_{a = (2y - 1)\gamma_p } \Trace\left( \left(\nabla_{\para} \model_\para(\feature_i) \right) \left(\nabla_{\para} \model_\para(\feature_i) \right)^\top\right) \notag \\
        =& \frac{1}{n} \frac{1}{p(1-p)} \sum_{i = 1}^n \|\nabla_{\para} \model_\para(\feature_i) \|_2^2. \label{eq:lemma_jacobian_log}
    \end{align}
    The proof is then complete.
\end{proof}

By~\Cref{lemma:sharpness_jacobian,lem:sharpness_jacobian_logistic}, we have that the sharpness yields by both loss functions are the same up to a constant factor. Therefore, the flattest minimizers of both loss functions are the same.

%% file: appendix_technical.tex
\subsection{Technical Lemmas}

\subsubsection{Concentration inequalities}
Subgaussian random variables are defined as follows.
\begin{definition}[Subgaussian random variable]
\label{def:subgaussian}
A random variable $X$ is called $\sigma$-subgaussian if $E[X] = 0$ and $\E\left[\exp\left(\lambda X\right)\right] \le \exp\left(\frac{\sigma^2 \lambda^2}{2}\right)$ for all $\lambda \in \R$.
\end{definition}

Subgaussian random vectors are defined as,
\begin{definition}[Subgaussian random vector]
\label{def:subgaussian_vector}
A random vector $x \in \R^d$ is called $\sigma$-subgaussian if $\E[x] = 0$ and $\E\left[\exp\left(\lambda^T x\right)\right] \le \exp\left(\frac{\sigma^2 \|\lambda\|_2^2}{2}\right)$ for all $\lambda \in \R^d$.
\end{definition}

We will further define subexponential random variables.
\begin{definition}[Subexponential random variable]
\label{def:subexponential}
A random variable $X$ is $(\sigma,\alpha)$-subexponential if $\E\left[\exp\left(\lambda (X - \E(X))\right)\right] \le \exp\left(\frac{\sigma^2 \lambda^2}{2}\right)$ for all $|\lambda| \le \frac{1}{\alpha}$.
\end{definition}
\begin{lemma}[\cite{honorio2014tight}]
\label{lem:subgaussian_subexponential}
If random variable $X$ is $\sigma$-subgaussian, then $X^2$ is $(4\sqrt{2}\sigma^2,4 \sigma^2)$-subexponential.
\end{lemma}

\begin{lemma}[Hoeffding's Bound]
\label{lem:subgaussian_concentration}
If $\{X_i\}_{i \in [n]}$ are $\sigma$-subgaussian and independent, then there exists $C_{\sigma} > 0$, for all $t \ge 0$, $\Pr \left(\left| \frac{1}{n} \sum_{i = 1}^n X_i \right| \ge t\right) \le 2\exp\left(-n t^2 C_{\sigma}\right)$.
\end{lemma}

\begin{lemma}[\cite{rinaldo2019probability}]
\label{lem:subexponential_concentration}
If $\{X_i\}_{i \in [n]}$ are $(\sigma,\alpha)$-subexponential and independent, then there exists $C_{\alpha, \sigma }> 0$, for all $t \ge 0$, $\Pr \left(\left| \frac{1}{n} \sum_{i = 1}^n (X_i - \E[X_i]) \right| \ge t\right) \le 2\exp\left(-n \min \left(tC_{\alpha, \sigma }, t^2 C_{\alpha, \sigma } \right)\right)$.
\end{lemma}

\begin{lemma}[\cite{rinaldo2019probability}]
\label{lem:norm_concentration}
If $x \in \R^d$ is a $\sigma$-Subgaussian random vector then for any $t \ge 0$,
\begin{align}
\label{eq:norm_concentration}
    \Pr\left(\|x\|_2^2 \ge 32\sigma^2 d + 8 \sigma^2 t  \right) \le \exp(-t).
\end{align}
It also holds that $\|x\|_2^2$ has bounded second moment $\E[\|x\|_2^4] \le 2048 \sigma^4 d^2$.
\end{lemma}

We will also need the following lemma bounding the truncated second-order moment of a subgaussian random variable.

\begin{lemma}
\label{lem:subgaussian_truncated_second_moment}
For any $n > 0$ and dimension $d$, for any $d$-dimension $\sigma$-subgaussian random vector $x$, there exists $c =O(\sqrt{d \log(dn)} \sigma)$, such that $\E\left[\|x\|^2 \one\left(\|x\| > c\right)\right] \le \frac{d}{n} \sigma^2$.
\end{lemma}

\begin{proof}
We have that by~\Cref{eq:norm_concentration}, 
\begin{align*}
    &\E\left[\|x\|^2 \one\left(\|x\| > c\right)\right] \\
    =& c^2 \Pr\left(\|x\|^2 > c^2\right) + 
    \int_{c}^{\infty}  \Pr\left(\|x\|^2 > t^2 \right) dt^2 \\
    \le& c^2 \exp(- \frac{c^2 - 32 \sigma^2d}{8 \sigma^2}) + \int_{c}^{\infty} 2t \exp(-\frac{t^2 - 32\sigma^2 d}{8\sigma^2}) dt \\
    =& c^2 \exp(- \frac{c^2 - 32 \sigma^2d}{8 \sigma^2}) + 8\sigma^2 \exp(-\frac{t^2 - 32\sigma^2 d}{8\sigma^2}) \mid^c_{\infty} \\ 
    \le&  c^2 \exp(- \frac{c^2 - 32 \sigma^2d}{8 \sigma^2}) + 8\sigma^2 \exp(-\frac{c^2 - 32\sigma^2 d}{8\sigma^2})
\end{align*}
Hence there exists $c = O(\sqrt{d \log (dn)} \sigma)$ such that $\E\left[\|x\|^2 \one\left(\|x\| > c\right)\right] \le \frac{d}{n}$.
\end{proof}

We will finally show a constant probability lower bound on the norm of a subgaussian random vector with unit variance.

\begin{lemma}
\label{lem:lower_bound}
Given any $\sigma > 0$, there exists constant $\epsilon, \zeta$, for any dimension $d$, for any $\sigma$-subgaussian random vector $x$ with connvariance $I_d$, it holds that $\Pr(\|x\|_2^2 > \epsilon d) > \zeta$.
\end{lemma}

\begin{proof}
    As $x$ is $\sigma$-subgaussian, it holds that for any $\lambda \in \R$,
    \begin{align}
    \label{eq:sym}
    \E_{\|v\| = 1}[\exp(\lambda v^\top x)] \le \exp(\lambda^2 \sigma^2 / 2).
    \end{align}
    Here the expectation over $v$ in~\Cref{eq:sym} is taken over a uniform distribution over a unit ball and $v$ is independent of $x$. Hence $v^\top x$ equals in distribution to $v[1] \|x\|_2$. Hence it holds that,
    \begin{align}
        \label{eq:sym_transform}
        \E_{\|v\| = 1}[\exp(v[1] \|x\|_2)] \le \exp(\sigma^2 / 2).
    \end{align}

    Note that $\exp(x) \ge 1 + x + \frac{x^2}{2} + \frac{x^3}{6} + \frac{x^4}{24} + \frac{x^5}{120}$ and $\forall t \in \mathbb{N}, \E[(v[1])^{2t + 1}] = 0$, it holds that
    \begin{align*}
        1 + \frac{1}{2} \E[\|x\|_2^2] \E_{\|v\| = 1}[(v[1])^2] + \frac{1}{24} \E[\|x\|_2^4] \E_{\|v\| = 1}[(v[1])^4] \le \exp(\sigma^2 / 2).
    \end{align*}

    It is well known that $ \E_{\|v\| = 1}[(v[1])^2] = \frac{1}{d}$ and $ \E_{\|v\| = 1}[(v[1])^4] = \frac{3}{(d + 2)(d + 4)}$. Also it holds that $\E[\|x\|_2^2] = d$.
    Hence,
    \begin{align*}
         \E[\|x\|_2^4] \le \left(\exp(\sigma^2 / 2) - 3/2\right)\frac{(d + 2)(d + 4)}{3}.
    \end{align*}
    
    This implies that 
    \begin{align*}
        &\left(\exp(\sigma^2 / 2) - 3/2\right)\frac{(d + 2)(d + 4)}{3} \\
        \ge&  \E[\|x\|_2^4 I(\|x\|_2^2 > \frac{1}{2}d)] \\
        \ge& \frac{\left(\E[\|x\|_2^2 I(\|x\|_2^2 > \frac{1}{2}d)]\right)^2}{\Pr(\|x\|_2^2 > \frac{1}{2}d)} \\
        =& \frac{\left(\E[\|x\|_2^2] - \E[\|x\|_2^2 I(\|x\|_2^2 \le \frac{1}{2}d)]\right)^2}{\Pr(\|x\|_2^2 > \frac{1}{2}d)} \\
        \ge& \frac{d^2}{4\Pr(\|x\|_2^2 > \frac{1}{2}d)}
    \end{align*}

    Hence, we can conclude that 
    \begin{align*}
        \Pr(\|x\|_2^2 > \frac{1}{2}d) \ge \frac{3d^2}{4(d + 2)(d + 4) \left(\exp(\sigma^2 / 2) - 3/2\right)} \ge \frac{1}{20\left(\exp(\sigma^2 / 2) - 3/2\right)}.
    \end{align*}
    This concludes the proof.
\end{proof}

\subsubsection{Rademacher Complexity}

Recall the definition of Rademacher complexity,
\begin{definition}[Rademacher complexity]
    \label{def:rademacher_complexity}
    Let $\mathcal{F}$ be a class of functions from $\mathcal{X}$ to $\mathcal{Y}$. Let $S = \{x_1,\ldots,x_n\} \subset \mathcal{X}$ be a set of points. The empirical Rademacher complexity of $\mathcal{F}$ with respect to $S$ is defined as $\gR_S(\gF) = \frac{1}{n}\Exp_{\epsilon \sim \{\pm 1\}^n} \sup_{f\in \gF}\sum_{i=1}^n  \epsilon_i f(x_i)$.
\end{definition}

We will also define the following notion of the shattered set and VC dimension.

\begin{definition}[Shattered set]
\label{def:shattered_set}
Let $\mathcal{F}$ be a class of functions from $\mathcal{X}$ to $\mathcal{Y}=\{0,1\}$. A set $S = \{x_1,\ldots,x_n\} \subset \mathcal{X}$ is said to be shattered by $\mathcal{F}$ if for every $T \subset S$, there exists $f \in \mathcal{F}$ such that $f(x) = 1$ for all $x \in T$ and $f(x) = 0$ for all $x \in S \setminus T$.
\end{definition}

\begin{definition}[VC dimension]
\label{def:vc_dimension}
Let $\mathcal{F}$ be a class of functions from $\mathcal{X}$ to $\mathcal{Y}=\{0,1\}$. The VC dimension of $\mathcal{F}$ is defined as $\VC(\mathcal{F}) = \sup \{n \in\mathbb{N}\mid \text{there exists a set of size $n$ shattered by $\mathcal{F}$}\}$.
\end{definition} 

We will use the following well-known lemmas.

\begin{lemma}[Massart's Lemma]
\label{lem:massart}
Let $\mathcal{F}$ be a class of functions from $\mathcal{X}$ to $\mathcal{Y} = \{0,1\}$. Further, suppose $\mathcal{A} = \{(f(x_i))_{i \in {n}} \mid f \in F\}$, then, $\gR_S(\gF) \le \sqrt{\frac{2 \log |\mathcal{A}|}{n}}$.
\end{lemma}

\begin{lemma}[Sauer's Lemma]
\label{lem:sauer}
Let $\mathcal{F}$ be a class of functions from $\mathcal{X}$ to $\mathcal{Y} = \{0,1\}$. Further, suppose $\mathcal{A} = \{(f(x_i))_{i \in [n]} \mid f \in F\}$, then $|\mathcal{A}| \le \sum_{i = 0}^{\VC(\mathcal{F})} {n \choose i}$. 
\end{lemma}

Combining the above two lemmas, we get the following corollary.
\begin{corollary}
\label{cor:rademacher_vc}
Let $\mathcal{F}$ be a class of functions from $\mathcal{X}$ to $\mathcal{Y} = \{0,1\}$, then $\gR_S(\gF) \le \sqrt{\frac{4 \VC(\mathcal{F}) \log n}{n}}$.
\end{corollary}

Further, we also have the following lemma controlling the VC dimension.
\begin{lemma}
\label{lem:vc_dimension}
Suppose $\mathcal{F} = \{x \in \R^d \to \one\left(w^\top x > 0 \right) \mid w \in \R^d \}$, then $\VC(\mathcal{F}) = d$.
\end{lemma}

The following uniform convergence bound based on Rademacher complexity is also well known.

\begin{lemma}[\cite{shalev2014understanding}]
\label{lem:uniform_convergence}
Suppose for all $f \in \gF$, $0 \le f(x) \le 1$, then with probability at least $1 - \delta$ over the randomness of i.i.d. sampled $S = \{x_1,\ldots,x_n\} \subset \mathcal{X}$, it holds that
\begin{align}
\sup_{f \in \gF} \left| \frac{1}{n}\sum_{i = 1}^n f(x_i) - \E[f(x)]\right| \le 2 \gR_S(\gF) + 3\sqrt{\frac{ \log \frac{4}{\delta}}{n}}.
\end{align}
\end{lemma}

To prove our main results, we will also need the following theorem due to~\cite{NIPS2010_76cf99d3}.

\begin{definition}
\label{def:smooth}
    A loss function $\ell: \R \times \R \to \R$ is $H-$smooth, if and only 
    $\frac{d \ell(x,y)}{dx}$ is $H-$lipschitz.
\end{definition}
\begin{theorem}[Theorem 1 of~\cite{NIPS2010_76cf99d3}]
\label{thm:complexity_loss}
For an $H$-smooth non-negative loss $\ell$ s.t. $\forall_{x, y, f}|\ell(f(x), y)| \leq b$, for any $\delta>0$ we have that with probability at least $1-\delta$ over a random sample of size $n$, for any $f \in \mathcal{F}$ with zero training loss $\hat{L}(h)=0$,
$$
L(h) \leq O\left(H \log ^3 n \mathcal{R}_n^2(\mathcal{F})+\frac{b \log (1 / \delta)}{n}\right).
$$
\end{theorem}

Finally, we also need lemmas bounding the Rademacher complexity of norm-bounded linear hypothesis and~\mlpnobias.

\begin{lemma}
\label{lem:radamacher_complexity_network}
For any constant $C > 0$ and number of samples $n$, for the set of parameters for~\mlpnobias~$\Theta_C \triangleq  \{ \theta \mid \sum_{j = 1}^m \| \weight_{1,j} \|_2 |\weight_{2,j}| \le C, \theta = (W_1,W_2)\}$ and any training set $\{x_i\}_{i \in [n]}$ satisfying that $\|x_i\|_2 \le B$, it holds that $\gR_S(\{  \modelnobias_{\theta} \mid \theta \in \Theta_C\}) \le \frac{2 C B}{\sqrt{n}}$.
\end{lemma}

\begin{proof}
Let $\overline{u}$ denotes $u / \|u\|_2$ for $u \neq 0$ and $0$ when $u = 0$, 
\begin{align*}
R_S(\{  \modelnobias_{\theta} \mid \theta \in \Theta_C\}) & =\frac{1}{n} \underset{\sigma}{\mathbb{E}}\left[\sup _\theta \sum_{i=1}^n \sigma_i \modelnobias_\theta\left(x_i\right)\right] \\
& =\frac{1}{n} \underset{\sigma}{\mathbb{E}}\left[\sup _\theta \sum_{i=1}^n \sigma_i\left[\sum_{j=1}^m \weight_{2,j} \ReLU\left(\weight_{1,j} x_i\right)\right]\right]  \\
& =\frac{1}{n} \underset{\sigma}{\mathbb{E}}\left[\sup _\theta \sum_{i=1}^n \sigma_i\left[\sum_{j=1}^m \weight_{2,j}\left\|\weight_{1,j}\right\|_2 \ReLU\left(\overline{W_{1,j}}^T x_i\right)\right]\right] \\
& =\frac{1}{n} \underset{\sigma}{\mathbb{E}}\left[\sup _\theta \sum_{j=1}^m \weight_{2,j}\left\|\weight_{1,j}\right\|_2\left[\sum_{i=1}^n \sigma_i \ReLU\left(\overline{W_{1,j}}^T x_i\right)\right]\right] \\
& \leq \frac{1}{n} \underset{\sigma}{\mathbb{E}}\left[\sup _\theta \sum_{j=1}^m\left|\weight_{2,j}\right|\left\|\weight_{1,j}\right\|_2 \max _{k \in[n]}\left|\sum_{i=1}^n \sigma_i \ReLU\left(\overline{W_{1,k}}^Tx_i\right)\right|\right] \\
& \leq  =\frac{C}{n} \underset{\sigma}{\mathbb{E}}\left[\sup _{\bar{u}:\|\bar{u}\|_2=1}\left|\sum_{i=1}^n \sigma_i \ReLU\left(\bar{u}^T x_i\right)\right|\right] \\
& \leq \frac{C}{n} \underset{\sigma}{\mathbb{E}}\left[\sup _{\bar{u}:\|\bar{u}\|_2 \leq 1}\left|\sum_{i=1}^n \sigma_i \ReLU\left(\bar{u}^T x_i\right)\right|\right] \\
& \leq \frac{2 C}{n} \underset{\sigma}{\mathbb{E}}\left[\sup _{\bar{u}:\|\bar{u}\|_2 \leq 1} \sum_{i=1}^n \sigma_i \ReLU\left(\bar{u}^T x_i\right)\right] \\
& =2 C R_S\left(\mathcal{H}{ }^{\prime}\right),
\end{align*}

where $\mathcal{H}^{\prime}=\left\{x \mapsto \ReLU\left(\bar{u}^{\top} x\right): \bar{u} \in \mathbb{R}^d,\|\bar{u}\|_2 \leq 1\right\}$. By Talagrand's lemma, since $\ReLU$ is 1-Lipschitz, $R_S\left(\mathcal{H}^{\prime}\right) \leq$ $R_S\left(\mathcal{H}^{\prime \prime}\right)$ where $\mathcal{H}^{\prime \prime}=\left\{x \mapsto \bar{u}^{\top} x: \bar{u} \in \mathbb{R}^d,\|\bar{u}\|_2 \leq 1\right\}$ is a linear hypothesis space. Using $R_S\left(\mathcal{H}^{\prime \prime}\right) \leq \frac{B}{\sqrt{n}}$ by~\Cref{lem:radamacher_complexity_linear} concludes the proof.
\end{proof}

\begin{lemma}
    \label{lem:radamacher_complexity_linear}
    For any constant $C > 0$ and number of samples $n$, for any set $S = \{x_i \}_{i \in [n]}$ satisfying that $\forall i, x_i \in \R^d, \|x_i\|_2^2 \leq C^2$ and function class $\mathcal{F}=\left\{x \mapsto\langle w, x\rangle \mid w \in \mathbb{R}^d,\|w\|_2 \leq 1\right\}$, it holds that,
    $$
    \gR_S(\mathcal{F}) \leq \frac{C}{\sqrt{n}}.
    $$
\end{lemma}
\begin{proof}
\begin{align*}
\gR_S(\mathcal{F}) & =\underset{\sigma}{\mathbb{E}}\left[\sup _{\|w\|_2 \leq 1} \frac{1}{n} \sum_{i=1}^n \sigma_i\left\langle w, x_i\right\rangle\right] \\
& =\frac{1}{n} \underset{\sigma}{\mathbb{E}}\left[\sup _{\|w\|_2 \leq 1}\left\langle w, \sum_{i=1}^n \sigma_i x_i\right\rangle\right] \\
& =\frac{1}{n} \underset{\sigma}{\mathbb{E}}\left[\left\|\sum_{i=1}^n \sigma_i x_i\right\|_2\right] \\
& =\frac{1}{n} \sqrt{\sum_{i=1}^n\left\|x_i\right\|_2^2} \le \frac{C}{\sqrt{n}}.
\end{align*}
\end{proof}

\subsubsection{Elementary inequalities}

We will prove some elementary inequalities that will be useful in the proof of our main results.

\begin{lemma}
\label{lem:relu_lipschitz}
For any $x, y \in \R$,
$|\ReLU(x)^2 - \ReLU(y)^2|\le (|x| + |y|) |x - y|$.
\end{lemma}
\begin{proof}
    We will assume WLOG that $x > y$. We then discuss the following three cases.
    \begin{enumerate}
        \vspace{-0.1in}
        \item $0 \ge x > y$, then the result is trivial.
        \item $x > 0 \ge y$, then $|\ReLU(x)^2 - \ReLU(y)^2| = x^2 \le (|x| + |y|) |x - y|$.
        \item $x > y > 0$, then $|\ReLU(x)^2 - \ReLU(y)^2| = x^2 - y^2 =   (|x| + |y|) |x - y|$.
    \end{enumerate}
    This completes the proof.
\end{proof}

\begin{lemma}
\label{lem:relu_bn}
For any $a, b, c, d \in \R$, if $a + d = b + c$, then
\begin{align*}
    | \relu(a) + \relu(d) - \relu(b) - \relu(c) |^2 \le \left(\relu(a)\right)^2 + \left(\relu(d)\right)^2 + \left(\relu(b)\right)^2 + \left(\relu(c)\right)^2.
\end{align*}
The equality holds if and only if three of the four values are not positive.
\vspace{-0.1in}
\end{lemma}

\begin{proof}
    WLOG we assume $a \ge b \ge c \ge d$. As ReLU is convex, we have that $\relu(a) + \relu(d) - \relu(b) - \relu(c) \ge 0$. Further, we have that $\relu(a) + \relu(d) - \relu(b) - \relu(c) \le \relu(a) - \relu(b) \le \relu(a)$. Thus, we have the desired result.
\end{proof}

%% file: appendix_experiment.tex
\vspace{-0.1in}
\begin{table}[h]
    \centering
    \begin{tabular}{|l|l|l|l|l|l|}
    \hline
                                                                         & Learning Rate & Perturbation Radius & Batch size & Weight Decay & Epochs \\ \hline
    \Cref{fig:classification_no_bias_baseline}       & 0.01          & 0                   & 100        & 0.05         & 1e5    \\ \hline
    \Cref{fig:classification_no_bias_sam}            &         0.01      &            0.05         &        1    &       0       &  1e5      
    \\ \hline
    \Cref{fig:classification_bn}            &         0.005      &            0.1         &        1    &       0       &  1e5       \\
    \hline
               &         0.003      &            1         &        1    &       0       &  1e5    
    \\ \hline
    \Cref{fig:classification_bias_baseline}           & 0.01          & 0                   & 100        & 0.05         & 1e5    \\
    \hline
    \Cref{fig:classification_bias_sam} & 0.01         & 0.05                   & 1          & 0         & 2e5    \\
    \hline
     & 0.01         & 0.1                   & 1          & 0         & 4e5    \\
    \hline
    \Cref{fig:classification_bias_softplus_baseline} & 0.001         & 0                   & 1          & 0.05         & 1e5   \\
    \hline
    \Cref{fig:classification_bias_softplus_sam} & 0.0005         & 0.05                   & 1          & 0         & 1e5   
    \\
    \hline
     & 0.001         & 0.1                   & 1          & 0         & 1e5   
    \\
    \hline
  & 0.005         & 1                   & 1          & 0         & 5e3    \\
  \hline
  \Cref{fig:classification_simple_ln_bias} & 0.1        & 0.1                   & 1          & 0         & 1e5  
  \\
  \hline
  \Cref{fig:classification_simple_ln_bias_norm_bounded} & 0.01        & 0.1                   & 1          & 0         & 5e2   \\
  \hline
   & 0.01        & 0.5                   & 1          & 0         & 5e2 \\
  \hline
   & 0.01        & 1                   & 1          & 0         & 5e2
   \\
  \hline
  \Cref{fig:regression_no_bias} & 0.01        & 0.2                   & 1          & 0         & 1e5
  \\
  \hline
  \Cref{fig:regression_bias} & 0.01        & 0.2                   & 1          & 0         & 1e5 
   \\
   \hline
  \Cref{fig:logistic_no_bias_baseline} & 0.1        & 0                   & 10          & 0.01         & 1e5 
   \\
   \hline
  \Cref{fig:logistic_no_bias_sam} & 0.1        & 0.2                   & 1          & 0         & 1e5 
   \\
   \hline
  \Cref{fig:logistic_bias_baseline} & 0.01        & 0                   & 1          & 0.05         & 1e5
  \\
   \hline
  \Cref{fig:logistic_bias_sam} & 0.1        & 0.2                   & 1          & 0         & 1e5 
  \\
  \hline
  \Cref{fig:logistic_simple_bn_bias} & 0.1        & 0.2                   & 1          & 0         & 4e4 
  \\
  \hline
  \Cref{fig:logistic_simple_ln_bias} & 1        & 0.5                   & 1          & 0         & 1e3 \\
  \hline
   & 1        & 1                   & 1          & 0         & 1e5
   \\
   \hline
  \Cref{fig:logistic_depth_3_bias_baseline} & 0.01        & 0                   & 1          & 0.05         & 1e5  
  \\
  \hline
  \Cref{fig:logistic_depth_3_bias_sam} & 0.01        & 0.05                   & 1          & 0         & 1e5 \\
    \hline
    \end{tabular}
    \caption{\textbf{Training details for Experiments.} For~\Cref{fig:classification_simple_ln_bias,fig:logistic_simple_ln_bias}, we scale down the initialization of the first layer by a factor of $100$ to avoid minimizing the sharpness by simply increasing the norm at the beginning. }
    \label{tab:details}
    \end{table}
\section{Experiments}
\label{sec:appendix_experiment}

\subsection{Training Details}

For all the experiments, we use networks with width $500$. The learning rates, perturbation radius, and training epochs are summarized in~\Cref{tab:details}. For those experiments where there are adjustments in hyperparameters through the training process, we report all the hyperparameters in multiple rows. We use 8 NVIDIA 2080 GPUs to train the models. The training time for each experiment is around 12 hours per 1e5 epochs

\subsection{Extension To Uniform Ball Distribution}
As our~\Cref{thm:sharpness_generalization_complexity_general,thm:sharpness_counter_example} suggests, the generalization and memorization results should hold for data distribution other than boolean hypercube. We perform experiments on uniform ball distribution to verify this. Specifically, we sample data points uniformly from the ball with radius $\sqrt{d}$ with dimension $d = 10$ and the label is defined as $y = |x[1]| - |x[2]|$. The results are shown in~\Cref{fig:regression_uniform_ball}. We can see that the flattest minimizers of the two architectures have very different generalization behavior. The flattest minimizer of the MLP without bias has a much better generalization performance than the one with bias. This is consistent with our theoretical results.

\begin{figure}
    \centering
    \begin{subfigure}[t]{0.45\textwidth}
        \centering
        \includegraphics[width=\textwidth]{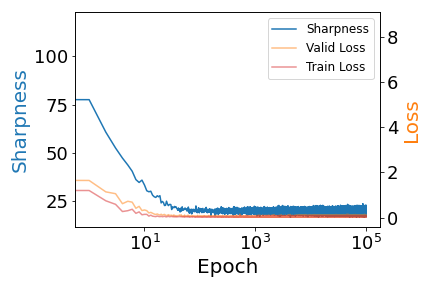}
        \caption{\mlpnobias}
        \label{fig:regression_no_bias}
    \end{subfigure}
    \begin{subfigure}[t]{0.45\textwidth}
       \centering
       \includegraphics[width=\textwidth]{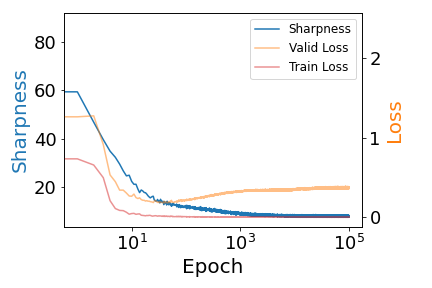}
       \caption{\mlpbias}
       \label{fig:regression_bias}
   \end{subfigure}
\caption{
\textbf{Uniform Ball Distribution.} We train a 2-layer MLP with ReLU activation with and without Bias using 1-SAM on uniform ball distribution with dimension $d = 10$ and training set size $n = 100$. One can again see the striking difference between the generalization behavior of the flattest minimizers of the two architectures.}
\label{fig:regression_uniform_ball}
\vspace{-0.1in}
\end{figure}

\subsection{Extension To Logistic Loss}
As~\Cref{thm:loss_minimizer,lem:sharpness_jacobian_logistic} suggests, our results can be extended to logistic loss with label smoothing. We perform all our experiments mentioned in the main text on the same distribution $\distr$, with the mean squared error loss replaced by logistic loss with label smoothing $p = 0.2$  to verify this. The results are shown in~\Cref{fig:logistic_no_bias,fig:logistic_bias,fig:logistic_normalization}.

\begin{figure}
    \centering
    \begin{subfigure}[t]{0.45\textwidth}
        \centering
        \includegraphics[width=\textwidth]{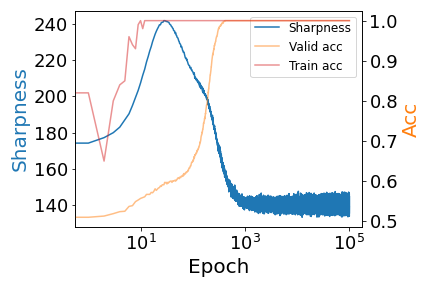}
        \caption{Baseline}
        \label{fig:logistic_no_bias_baseline}
    \end{subfigure}
    \begin{subfigure}[t]{0.45\textwidth}
       \centering
       \includegraphics[width=\textwidth]{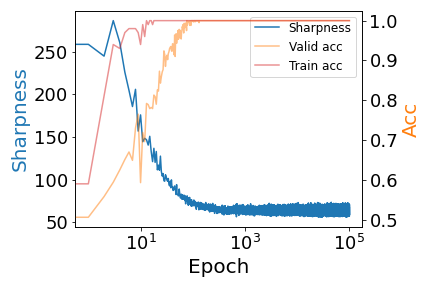}
       \caption{1-SAM}
       \label{fig:logistic_no_bias_sam}
   \end{subfigure}
\caption{
\textbf{Scenario I with Logistic Loss.} We train a 2-layer MLP with ReLU activation without Bias using gradient descent with weight decay and 1-SAM on $\distr$ with dimension $d = 30$ and training set size $n = 100$. In both cases, the model reaches perfect generalization. Notice that although weight decay doesn't explicitly regularize model sharpness, the flatness of the model decreases through training, which is consistent with our~\Cref{lem:norm_bound} relating sharpness to the norm of the weight.}
\label{fig:logistic_no_bias}
\vspace{-0.2in}
\end{figure}

\begin{figure}
    \centering
    \begin{subfigure}[t]{0.45\textwidth}
        \centering
        \includegraphics[width=\textwidth]{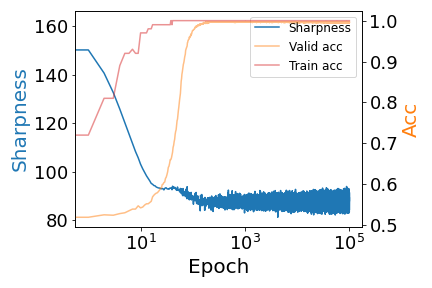}
        \caption{Baseline}
        \label{fig:logistic_bias_baseline}
    \end{subfigure}
    \begin{subfigure}[t]{0.45\textwidth}
       \centering
       \includegraphics[width=\textwidth]{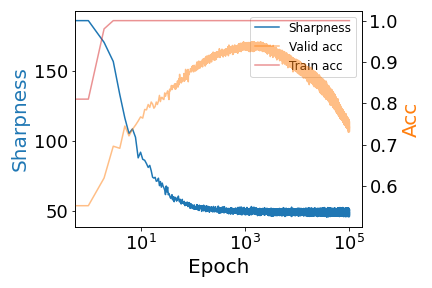}
       \caption{1-SAM}
       \label{fig:logistic_bias_sam}
   \end{subfigure}
\caption{
\textbf{Scenario II with Logistic Loss.} We train a 2-layer MLP with ReLU activation with Bias using gradient descent with weight decay and 1-SAM on $\distr$ with dimension $d = 30$ and training set size $n = 100$. One can observe a distinction between the two settings. The minimum reached by 1-SAM is flatter but the model generalizes much worse and even starts to degenerate after 2000 epochs. The difference between~\Cref{fig:logistic_no_bias_sam,fig:logistic_bias_sam} is similar to the difference between~\Cref{fig:classification_no_bias_sam,fig:classification_bias_sam}}
\label{fig:logistic_bias}
\end{figure}

\begin{figure}
    \centering
    \begin{subfigure}[t]{0.45\textwidth}
        \centering
        \includegraphics[width=\textwidth]{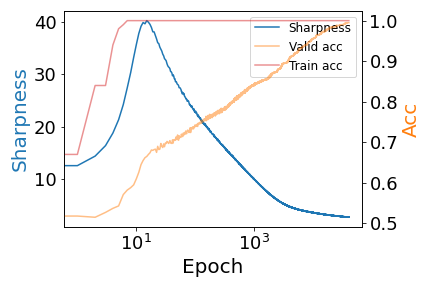}
        \caption{Simplified BatchNorm}
        \label{fig:logistic_simple_bn_bias}
    \end{subfigure}
    \begin{subfigure}[t]
        {0.45\textwidth}
        \centering
        \includegraphics[width=\textwidth]{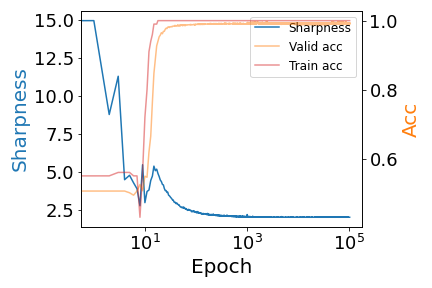}
        \caption{Simplified LayerNorm}
        \label{fig:logistic_simple_ln_bias}
    \end{subfigure}
    \caption{\textbf{Models with Normalization and Logistic Loss.} We train two-layer ReLU networks with simplified BatchNorm and LayerNorm on data distribution $\distr$ with dimension $d = 30$ and sample complexity $n = 100$ using 1-SAM. We can see that in both cases, the models nearly perfectly generalize.}
\label{fig:logistic_normalization}
\vspace{-0.3cm}
\end{figure}

\begin{figure}
    \centering
    \begin{subfigure}[h]{0.45\textwidth}
        \centering
        \includegraphics[width=\textwidth]{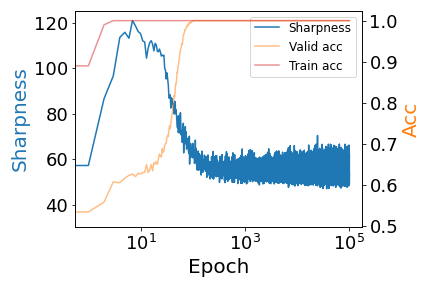}
        \caption{Baseline}
        \label{fig:logistic_depth_3_bias_baseline}
    \end{subfigure}
    \begin{subfigure}[h]{0.45\textwidth}
       \centering
       \includegraphics[width=\textwidth]{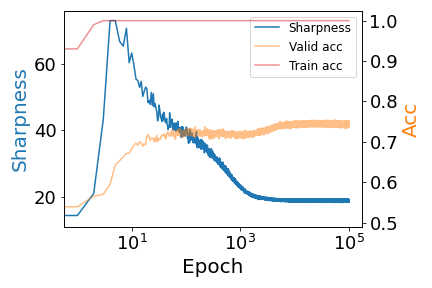}
       \caption{1-SAM}
       \label{fig:logistic_depth_3_bias_sam}
   \end{subfigure}
\caption{
\textbf{Scenario II with Deeper Networks.} We train a 3-layer MLP with ReLU activation with Bias using gradient descent with weight decay and 1-SAM on $\distr$ with dimension $d = 30$ and training set size $n = 100$. One can observe a distinction between the two settings. The minimum reached by 1-SAM is flatter, but the model generalizes much worse.}
\label{fig:logistic_bias_depth_3}
\end{figure}

\subsection{Extension To Deeper Networks}

Our~\Cref{thm:sharpness_counter_example} suggests that memorization solutions can exist for deeper networks with biased terms in the first layer. We perform experiments on deeper networks to verify this. Specifically, we train a 3-layer MLP with ReLU activation with bias term in the first layer on $\distr$ with dimension $d = 30$ and training set size $n = 100$. The results are shown in~\Cref{fig:logistic_bias_depth_3}. We can see that the flattest minimizer of the 3-layer MLP with bias term in the first layer has a much worse generalization performance than the baseline. This is consistent with our theoretical results.